\documentclass[11pt]{amsart}
\usepackage[T1]{fontenc}
\usepackage{amsmath,amsfonts,amsbsy,amsgen,amscd,mathrsfs,amssymb,amsthm}
\usepackage{enumerate,mathtools}
\usepackage{xcolor}
\usepackage{hyperref}
\usepackage[capitalize,noabbrev]{cleveref}
\usepackage{natbib}

\hypersetup{
  colorlinks=true,
  linkcolor=red,
  citecolor=blue,
  urlcolor=black
}

\makeatletter
\renewcommand{\eqref}[1]{%
  \hyperref[#1]{\textup{\tagform@{\ref*{#1}}}}%
}
\makeatother

\usepackage{float}
\usepackage{subcaption}
\usepackage{siunitx}

\usepackage{xcolor}

\usepackage{amssymb}
\usepackage{graphics}
\usepackage[mathcal]{euscript}
\usepackage{indentfirst}
\usepackage{braket}
\usepackage{amssymb}
\usepackage{latexsym}
\usepackage{amsmath}
\usepackage{amsthm}
\usepackage{amscd}
\usepackage{enumerate}
\usepackage{graphicx}
\usepackage{fullpage}
\usepackage{url}
\usepackage{tikz}
\usepackage{mathrsfs}
\usepackage{algorithm}
\usepackage{algorithmic}
\usepackage{booktabs}

\numberwithin{equation}{section}

\newtheorem{theorem}{Theorem}[section]
\newtheorem{proposition}[theorem]{Proposition}

\newtheorem{corollary}[theorem]{Corollary}
\newtheorem{lemma}[theorem]{Lemma}

\theoremstyle{definition}
\newtheorem{definition}[theorem]{Definition}
\newtheorem{claim}[theorem]{Claim}

\newtheorem{remark}[theorem]{Remark}

\newtheorem*{claim*}{Claim}

\newcommand{\Def}{\textbf}
\newcommand*\diff{\mathop{}\!\mathrm{d}}
\newcommand{\target}{\mu}

\newcommand{\flow}{p}
\newcommand{\trans}{\flow}
\newcommand{\Space}{\mathcal O}
\newcommand{\Sett}{\mathbb X}
\newcommand{\Setsig}{\mathcal X}
\newcommand{\Time}{T}
\newcommand{\x}{x}
\newcommand{\h}{h}
\newcommand{\R}{\mathbb R}
\newcommand{\N}{\mathbb N}
\newcommand{\Z}{\mathbb Z}
\newcommand{\Dim}{d}
\newcommand{\den}{\mathrm m}
\newcommand{\E}{\mathbb E}
\newcommand{\PP}{\mathbb P}
\newcommand{\semi}{\mathrm P}
\newcommand{\pos}{\pi}
\newcommand{\reltarget}{f}
\newcommand{\fn}{G}
\newcommand{\cvx}{\ell}
\newcommand{\beg}{\mathrm D_{\cvx}}
\newcommand{\B}{\operatorname{Binomial}}
\newcommand{\Poisson}{\operatorname{Poisson}}
\newcommand{\NB}{\operatorname{Negative Binomial}}
\newcommand{\var}{\operatorname{Var}}
\newcommand{\cov}{\operatorname{Cov}}
\newcommand{\Bin}{\mathrm B}

\newcommand{\ints}{\lambda}
\newcommand{\idx}{i}

\newcommand{\cp}{X}
\newcommand{\XX}{\cp}
\newcommand{\der}{\partial}
\newcommand{\basisi}{e_{\idx}}
\newcommand{\basis}{e}

\newcommand{\Leb}{\mathcal L}
\newcommand{\ub}{C}
\newcommand{\fp}{Y}
\newcommand{\bm}{B}
\newcommand{\heat}{H}
\newcommand{\gauss}{\gamma}
\newcommand{\id}{I}
\newcommand{\Rp}{R}
\newcommand{\ent}{\mathrm {KL}}
\newcommand{\Qp}{Q}
\newcommand{\paths}{D([0,\Time];\N^{\Dim})}

\newcommand{\loss}{\mathsf{L}}
\newcommand{\w}{w}

\newcommand{\cev}[1]{\reflectbox{\ensuremath{\vec{\reflectbox{\ensuremath{#1}}}}}}

\title{Binomial flows: Denoising and flow matching for discrete ordinal data}
\author{Yair Shenfeld}
\address{Division of Applied Mathematics, Brown University, Providence, RI, USA 02912}
\email{yair\_shenfeld@brown.edu}
\author{Ricardo Baptista}
\address{Department of Statistical Sciences, University of Toronto, Toronto, ON, Canada M5G 1X6}
\email{r.baptista@utoronto.ca}
\author{Stefano Peluchetti}
\address{Sakana AI, Tokyo, Japan}
\email{stepelu@sakana.ai}

 \thanks{}

\begin{document}
\maketitle

\begin{abstract}
Flow-based generative modeling in continuous spaces exploit
Tweedie's formula to express the denoiser (learned in training) as a score function (used in sampling). In contrast, this relation has been largely missing in the discrete setting where common approaches focus on learning discrete scores and rates. In this work we close this gap for discrete non-negative ordinal data by introducing Binomial flows. Our framework provides a simple recipe for training a discrete diffusion model which simultaneously denoises, samples, and estimates exact likelihoods. We verify our methodology on synthetic examples and obtain competitive results on real-world data sets.
\end{abstract} 

\section{Introduction}

Flow matching and diffusion models in continuous spaces are currently state-of-the-art methods for the generation of images, video, audio, proteins, and robotic data~\cite{DBLP:journals/corr/abs-2412-06264}. Given that many of these structured dataset are naturally defined over discrete spaces, e.g., using 8-bit representations of images, text, molecules~\cite{liu2025think}, there is active research in bringing these flow-based models into the discrete setting. 
While many of the ideas of flow-based methods carry over from continuous to discrete settings, there are in fact some important differences and missing gaps. In continuous spaces, the standard approaches perturb data points with additive noise, train a denoiser to recover the clean data points from their noisy versions, and then sequentially apply denoising steps to sample new data points. A key relation for using denoisers to solve the sampling task correctly is \emph{Tweedie's formula}, which expresses the denoiser as the score function of the perturbed data distribution. The score is used to sample using a variety of iterative techniques, such as Langevin dynamics or time-reversed diffusion-based solvers, which progressively transform noise into samples from the target data distribution. 

In contrast, flow-based methods in discrete spaces have mostly focused on  learning a (discrete analogue of the) score. This is partly due to the fact that \emph{Gaussian} noise, which is the natural choice for perturbing the data
in continuous settings, satisfies a Tweedie's formula relating denoisers and scores, while the analogue of Gaussian noise in the discrete setting is less clear. As argued in recent works \cite{pham2025discrete, karras2022elucidating, li2025back}, learning denoisers, as opposed to  scores, improves stability in training so it is of interest to find a discrete analogue to Tweedie's formula. In this paper we show that when the data is defined over an ordered discrete space with nonnegative data, \emph{Binomial} noise is the analogue of Gaussian noise. Based on this observation we introduce a simple recipe for training a discrete diffusion model which mirrors diffusion models in continuous spaces.

Our main contributions are as follows:
\begin{enumerate} \itemsep0pt
\item We show that choosing the noising distributions to be \Def{Binomial} allows us to directly learn the denoiser, which in turn satisfies a discrete Tweedie's formula.
\item We show that the discrete Tweedie's formula yields exact samples from the data distribution in a finite time through the \Def{Poisson-F\"ollmer process}.

\item We derive an \emph{identity} for the likelihood of the data distribution involving the Binomial flow denoiser, which leads to \Def{exact likelihood estimation and training}. 
\item We validate our methodology on various 
synthetic and real-world 
datasets and obtain competitive FID values on imaging datasets. 
\end{enumerate}
In the remainder  of the introduction, we review in more detail 
the general framework of flow matching and diffusion models. 
Section \ref{sec:Bin_flow} presents our main results on the Binomial flows and the Poisson-F\"ollmer process. Section \ref{sec:exp} describes our experiments. Section \ref{sec:related} covers the relevant related work. The proofs of our results and further information are contained in  Appendices \ref{sec:poisson_semigroup}-\ref{sec:experiments}.

\subsection{Learning denoisers}
Consider a data distribution $\target$ over an ordered space $\Space$, e.g, $\Space=\R^{\Dim}$ or $\Space=\N^{\Dim}$. Our goal is to learn a denoiser 
\begin{equation}
\label{eq:denoiser_intro}
\den(t,\x):=\E[\XX_{\Time}|\XX_t=\x]
\end{equation}
for all $t\in [0,\Time]$ and $\x\in\Space$, where $\XX_t$ is a noisy version of $\XX_{\Time} \sim \target$ with $t$ standing for the signal level. A flow-based recipe to train a denoiser is to choose a family of conditional distributions $\flow_{t|\Time}$ on $\Space$, for $t\in [0,\Time]$, such that  
\begin{enumerate}[(i)] 
\item \label{enm:easy} \itemsep0pt $\flow_{t|\Time}(\cdot|\x_{\Time})$ can be easily sampled for all $t\in [0,\Time]$ and $\x_{\Time}\in\Space$,
\item \label{enm:delta} $\flow_{\Time|\Time}(\x_{\Time}|y)=\delta_y(\x_{\Time})$ for all $\x_{\Time},y\in\Space$,
\item \label{enm:ind} $\flow_{0|\Time}(\cdot|\x_{\Time})$ is independent of $\x_{\Time}$ for all $\x_{\Time}\in\Space$.
\end{enumerate}
Property \eqref{enm:easy} means that it is easy to perturb a data point $\x_{\Time}$, property \eqref{enm:delta} means that at time $t=\Time$, when there is no perturbation, $x_t$ corresponds to the clean data. Property \eqref{enm:ind} means that at $t=0$ the data point $\x_{\Time}$ is forgotten and $x_0$ is pure noise. Once the family $(\flow_{t|\Time})_{t\in [0,\Time]}$ is constructed, we define the flow $(\flow_t)_{t\in [0,\Time]}$ of probability measures as
\begin{equation}
\label{eq:marg_eq_intro}
\flow_t(\x_t):=\int_{\Space} \flow_{t|\Time}(\x_t|\x_{\Time}) \target(\x_{\Time}).
\end{equation}
By construction, $\flow_{\Time}=\target$, and $\flow_0$ is a measure which is independent of the data, typically assumed to be easy to sample. Thus, the flow $(\flow_t)_{t\in [0,\Time]}$ interpolates between an easy distribution at $t=0$ to the data distribution $\target$ at $t=\Time$. To learn the denoiser $\den$ in \eqref{eq:denoiser_intro}, with $\XX_t\sim \flow_t$, we minimize
\begin{equation}
\label{eq:denoiser_obj_intro}
\min_{\theta}\int_0^{\Time}\int_{\Space} w(t) \beg\left(\x_{\Time},\den_{\theta}(t,\x_t)\right) \flow_{t|\Time}(\x_t|\x_{\Time})\target(\x_{\Time})\diff t,
\end{equation}
where $\beg$ is a Bregman divergence for a convex function $\cvx$ (e.g., quadratic, entropic), $\w \colon [0,T] \rightarrow \R_{+}$ is a time-dependent weight, and $\{\den_{\theta}\}_{\theta}$ is a class of functions parameterized by $\theta$. For example, with $\cvx(\x)=\frac{|\x|^2}{2}$, the optimization problem~\eqref{eq:denoiser_obj_intro} reads 
\begin{equation}
\label{eq:denoiser_obj_quad_intro}
\min_{\theta}\int_0^{\Time}\int_{\Space} w(t) \left|\x_{\Time}-\den_{\theta}(t,\x_t)\right|^2 \flow_{t|\Time}(\x_t|\x_{\Time})\target(\x_{\Time})\diff t.
\end{equation}
A standard computation \cite{banerjee2005optimality}, \cite{DBLP:journals/corr/abs-2505-05082} shows that $\den(t,\x)=\E[\XX_{\Time}|\XX_t=\x]$ is the unique minimizer of \eqref{eq:denoiser_obj_intro} for any strictly convex $\cvx$, provided that $\{\den_{\theta}\}_{\theta}$ is a sufficiently large class of functions. The advantage of the objective~\eqref {eq:denoiser_obj_intro} is that, given i.i.d.\ samples from $\target$, it can be estimated with Monte Carlo, since by property \eqref{enm:easy} sampling from $\flow_{t|\Time}(\cdot|\x_{\Time})$ is easy. 
\subsection{Sampling with denoisers}
In the context of generative modeling one is interested in using a denoiser $\den_{\theta}$ to generate new samples from $\target$. In the continuous setting, e.g., $\Space =\R^{\Dim}$, this can be achieved by choosing the conditional distributions to be \emph{Gaussian}. For example, we can choose
\begin{equation}
\label{eq:gaussian_cond_intro}
\flow_{t|\Time}(\cdot|\x_{\Time}):=\mathcal N\left(\frac{t}{\Time}\x_{\Time},\frac{t(\Time-t)}{\Time}\id_{\Dim}\right),
\end{equation}
which indeed satisfies \eqref{enm:easy}-\eqref{enm:delta}-\eqref{enm:ind}\footnote{With the choice \eqref{eq:gaussian_cond_intro} we have $\flow_{0|\Time}(\cdot|\x_{\Time})=\delta_0$,  where other Gaussian models would lead to, for example, $\flow_{0|\Time}(\cdot|\x_{\Time})\approx \mathcal N(0, \id_{\Dim})$. In our exposition we chose  \eqref{eq:gaussian_cond_intro} to make the analogy with the discrete case clearer; cf.  \cref{sec:gaussian}.}. In order to sample we use \Def{Tweedie's formula},
\begin{equation}
\label{eq:Tweedie_cont_intro}
\den(t,x)=\frac{\Time}{t}\x+(\Time-t)\,\nabla\log \flow_t(\x)
\end{equation}
to express the denoiser $\den(t,x)$ in terms of the score  $\nabla\log \flow_t(\x)$. Once we have an estimate of $\nabla\log \flow_t$ we can sample from $\target$ via a number of methods such as time-reversal, (annealed) Langevin dynamics, and probability flows \cite{ho2020denoising, song2021scorebased, song2019generative}.

The setting of discrete data posses a challenge to implementing the above approach, due to the lack of the Gaussian distribution, and the lack of ordinal structure. Instead, most approaches focus on learning the distributions $\trans_{T|t}(\cdot|\x_t)$, which predicts clean data from its noisy versions, and then use these distributions in sampling, often in heuristic ways. Since the denoiser $\den(t,\x_t)$ is simply the mean of $\trans_{T|t}(\cdot|\x_t)$, it is much easier to learn, so being able to get exact samples using just the denoiser is of great value. This is precisely the goal we achieve in this work.

\section{Binomial flows}
\label{sec:Bin_flow}
\subsection{Binomial denoisers}
\label{subsec:bin_denoiser}
Let $\x_{\Time}\in \N^{\Dim}$ be a clean data point which we wish to noise into $\x_t\in \N^{\Dim}$ for $t\in [0,\Time]$. In the continuous setting the natural and common approach is to use \emph{additive} noise to perturb the data, $\x_t=\x_{\Time}+(\Time -t)y$ where $y$ is the noise. But in the discrete setting such perturbation will yield $x_t\notin \N^{\Dim}$. %
Our approach is to use \emph{thinning} as a substitute to addition. In particular, we thin $\x_{\Time}\in \N^{\Dim}$ by sampling $x_t|x_{\Time}$ from a Binomial (product) distribution with the number of trials $\x_{\Time}$ and success probability $\frac{t}{\Time}$. Clearly, at $t=\Time$ we recover the clean data $\x_{\Time}$, while at $t=0$ we have no signal left, $\x_0=0$. We now make these definitions precise.

Let $\B_{n,\alpha}$ be the binomial distribution with $n \in \N$ trials and probability of success $\alpha\in [0,1]$,
\[
\B_{n,\alpha}(k)={n\choose k}\alpha^k(1-\alpha)^{n-k},\quad k=0,\ldots, n.
\]
For $\x_{\Time}\in \N^{\Dim}$ let $\B_{\x_{\Time},\alpha}$ stand for the product distribution that factorizes over each component
\[
\B_{\x_{\Time},\alpha}=\prod_{\idx=1}^{\Dim}\B_{\x_{\Time}^{\idx},\alpha},\quad \x_{\Time}=(x_{\Time}^1,\ldots,\x_{\Time}^{\Dim}).
\]
Given a data sample $\x_{\Time} \in \N^{\Dim}$ we define the \Def{Binomial flow} as the family of conditional distributions for the thinned data,
\begin{equation}
\label{eq:cond_dist_bino_poss}
\flow_{t|\Time}(\x_t|\x_{\Time}):=\B_{\x_{\Time},\frac{t}{\Time}}(\x_t).
\end{equation}
We note that the choice \eqref{eq:cond_dist_bino_poss} satisfies properties \eqref{enm:easy},  \eqref{enm:delta}, \eqref{enm:ind}, with $\flow_{0|\Time}(\x_0|\x_{\Time})=\delta_0$. The \Def{denoiser} is given by
\begin{equation}
\label{eq:denoiser_def}
\den(t,\x):=\E_{\XX_{\Time} \sim p_{T|t}(\cdot|\XX_t=\x)}[\XX_{\Time}|\XX_t=\x],
\end{equation}
where $\trans_{\Time|t}(\cdot|\XX_t) \propto \trans_{t|\Time}(\XX_t|\cdot)\target(\cdot)$ 
and $\XX_t|\XX_{\Time}\sim \flow_{t|\Time}(\cdot|\XX_{\Time})=\B_{\XX_{\Time},\frac{t}{\Time}}$.
The denoiser $\den$ can be estimated by solving~\eqref{eq:denoiser_obj_intro} for the Binomial flow as shown in Algorithm~\ref{alg:train}.
\begin{algorithm}[tb]
  \caption{Training denoiser: Binomial flow}
  \label{alg:train}
  \begin{algorithmic}
    \STATE \textbf{Input:} Dataset $\mathcal{D}$, final time $\Time>0$, Bregman divergence $\beg$, batch size $B$, family of denoisers $\{\den_{\theta}\}_{\theta}$
         \WHILE{training}
         \STATE Sample minibatch $\{\x_{\Time}^{(k)}\}_{k=1}^B\sim \mathcal{D}$
         \STATE Sample times $\{t^{(k)}\}_{k=1}^B$  in $[0,\Time]$
         \STATE Sample $\x_t^{(k)}\sim \B_{\x_{\Time}^{(k)},\frac{t^{(k)}}{\Time}}$ for $k=1,\ldots,B$
      \STATE Compute:
      $$g_{\theta}:=\nabla_{\theta}\sum_{k=1}^B\beg\left(\x_{\Time}^{(k)},\den_{\theta}(t^{(k)},\x_t^{(k)})\right)$$
      \STATE Update: $\theta \leftarrow \texttt{Optimizer}(\theta,g_{\theta})$.
      \ENDWHILE
      \end{algorithmic}
\end{algorithm}

Once the denoiser in \eqref{eq:denoiser_def} is estimated we need to find a way to use it in sampling. In the following sections we will show that there exists a continuous time Markov chain $(\cp_t)_{t\in [0,T]}$ in $\N^{\Dim}$, the \Def{Poisson-F\"ollmer process},  with  intensity  $\ints(t,\x)$ such that 
\begin{enumerate}[(a)] \itemsep0pt
\item \label{enm:X_T_target} $\cp_{\Time}\sim \target,\qquad \cp_0=0$,
\item \label{enm:bridge} $\cp_t|\cp_{\Time}\sim \flow_{t|\Time}(\cdot|\cp_{\Time})=\B_{\cp_{\Time},\frac{t}{\Time}}$,
\item \label{enm:Tweedie} $\frac{\den(t,\cp_t)-\cp_t}{\Time-t}=\ints(t,\cp_t)$.
\end{enumerate}
The combination of \eqref{enm:X_T_target}-\eqref{enm:bridge}-\eqref{enm:Tweedie} is the precise discrete analogue of the Gaussian setting. Property~\eqref{enm:X_T_target} guarantees that the process matches the data distribution at the final time $t = T$, property~\eqref{enm:bridge} allows us to estimate the denoiser $\den$ by solving~\eqref{eq:denoiser_obj_intro}, and property \eqref{enm:Tweedie} is a \Def{discrete Tweedie's formula}, which shows how to use the denoiser to sample $(\cp_t)_{t\in [0,T]}$ via Algorithm~\ref{alg:sample}. 
\begin{algorithm}[tb]
  \caption{Sampling the Poisson-F\"ollmer process}
  \label{alg:sample}
  \begin{algorithmic}
    \STATE \textbf{Input:} Denoiser $\den_{\theta}$, final time $\Time>0$, step size $\Delta t$
    \STATE \textbf{Initialize:} $\cp_0:=0$
    \STATE \textbf{Set:} rate $\ints_{\theta}(t,\x):=\frac{\den_{\theta}(t,\x)-\x}{\Time-t}$
    
    \FOR{$k=1,\ldots,\frac{\Time}{\Delta t}$}
      \STATE $\cp_{k\Delta t} = \texttt{Sampler}(\cp_{(k-1)\Delta t},\ints_{\theta},\Delta t)$
    \ENDFOR
    \STATE \textbf{Output:} $\cp_{\Time}$
  \end{algorithmic}

  \vspace{0.35em}
  \hrule
  \vspace{0.35em}
  \footnotesize
  \noindent The \texttt{Sampler} method can be implemented with Algorithms \ref{alg:euler} or \ref{alg:tau}.
\end{algorithm}
In the next section we formally introduce the process $(\cp_t)_{t\in [0,T]}$. It will be shown to possess desirable theoretical properties, and in addition will yield an identity to compute the \Def{likelihood} of the data under this process via a particular choice of a Bregman divergence.
\subsection{The Poisson-F\"ollmer process}
\label{subsec:PF}
The Poisson-F\"ollmer process $(\cp_t)_{t\in [0,T]}$ is a continuous-time Markov chain (CTMC) with values in $\N^{\Dim}$. We refer to \citep{DBLP:journals/corr/abs-2205-14987} for a quick introduction to CTMC in the context of discrete diffusion models, and recall that the rate matrix $R_t(x,y)_{x,y\in\Space}$ completely describes the process. The rate matrix for  the Poisson-F\"ollmer process is 
\begin{equation}
\label{eq:rate_pf}
R_t(x,y)=
\begin{cases}
\ints^{\idx}(t,x) & \textnormal{if } y = x+e_{\idx},\\
1-\sum_{i=1}^{\Dim}\ints^{\idx}(t,x) & \textnormal{if } y = x,\\
0 & \textnormal{else},
\end{cases}
\end{equation}
where $\{\basisi\}_{\idx=1}^{\Dim}$ are the one-hot vectors and $\{\ints^{\idx}(t,x)\}_{\idx=1}^{\Dim}$ are the \Def{intensities} defined below. In words, The Poisson-F\"ollmer process $(\cp_t)$ is a continuous-time counting-process where at time $t$ each coordinate $\cp_t^{\idx}$ jumps up to $\cp_t^{\idx}+1$ with intensity $\ints^{\idx}(t,\cp_t)$. Note that the decision of each coordinate to jump depends on the current values  of all the other coordinates since the intensity $\ints^{\idx}(t,\cp_t)$ of the $\idx$th coordinate depends on $\cp_t$, rather than just $\cp_t^{\idx}$. The intensities will be chosen in such a way that we are guaranteed that $\cp_{\Time}$ is distributed like the data distribution $\target$.

To define the intensity
\begin{equation}
\label{eq:ints_def}
\ints(t,x):=(\ints^1(t,x),\ldots, \ints^{\Dim}(t,x)),
\end{equation}
we start with the definition of the one-dimensional Poisson distribution $\pos_t$ with parameter $t\ge 0$,
\begin{equation}
\label{eq:pos_def}
\pos_t(k)=e^{-t}\frac{t^k}{k!},\qquad k\in \N,
\end{equation}
and its multi-dimensional extension, 
\begin{equation}
\label{eq:pos_multi_def}
\pos_t(x)=\prod_{j=1}^{\Dim}\pos_t(\x^j),\qquad x=(\x^1,\ldots,\x^{\Dim})\in \N^{\Dim}.
\end{equation}
Based on the Poisson distribution we define the Poisson semigroup (which is analogous to the convolution operator/heat semigroup defined by the Gaussian distribution).

\begin{definition}[The Poisson semigroup]
Let $\fn:\N^{\Dim}\to \R$ be such that $\sum_{\x\in \N^{\Dim}}|\fn(\x)|\pos_t(x)<+\infty$. The \emph{Poisson semigroup} $(\semi_t)_{t\ge 0}$ is given by 
\begin{equation}
\label{eq:Poss_semi_def}
\semi_t\fn(x):=\sum_{y\in \N^{\Dim}}\fn(\x+y)\pos_t(y).
\end{equation}
\end{definition}
The intensities are now defined by 
\begin{equation}
\label{eq:ints_j_def}
\ints^{\idx}(t,x):=\frac{\semi_{\Time-t}\reltarget(x+\basisi)}{\semi_{\Time-t}\reltarget(x)},\quad t\in [0,\Time],\quad x\in \N^{\Dim},
\end{equation}
where $\{\basisi\}_{\idx=1}^{\Dim}$ are the one-hot vectors, and where $\reltarget$ is the ratio of the data distribution $\target$ and $\pos_{\Time}$,
\begin{equation}
\label{eq:rel_den}
\reltarget(\x):=\frac{\target(x)}{\pos_{\Time}(\x)},\qquad \x\in \N^{\Dim}.
\end{equation}
Once the intensities are defined as in \cref{eq:ints_j_def} one can verify that $\cp_{\Time}$ follows the data distribution $\target$ (\cref{cor:time_marg}). The form of the intensities in \eqref{eq:ints_j_def} is the discrete analogue of the score in the continuous setting (\cref{sec:gaussian}), and similarly cannot be computed simply from the definition \eqref{eq:ints_j_def}. However, in the next section we will overcome this issue using our discrete Tweedie's formula.
\subsection{Binomial denoisers and the Poisson-F\"ollmer process}
Our first key result is that learning the Binomial denoiser allows us to sample the data distribution using the Poisson-F\"ollmer process. 

\label{subsec:main_results}
\begin{proposition}
\label{prop:key_properties}
Fix $\Time>0$ and let $\target=\reltarget\pos_{\Time}$, with $\reltarget:\N^{\Dim}\to (0,\infty)$ bounded from above and below by positive constants\footnote{These assumptions are innocuous in practice as positivity can be guaranteed by setting the probability of any specific $x\in \N^{\Dim}$ to be arbitrarily small, while the upper bound assumption can be guaranteed with a large enough upper bound.}, be a distribution on $\N^{\Dim}$. The Poisson-F\"ollmer process  $(\cp_t)_{t\in [0,T]}$ with the intensities \eqref{eq:ints_j_def} satisfies \eqref{enm:X_T_target}-\eqref{enm:bridge}-\eqref{enm:Tweedie},
\begin{enumerate}[(a)] \itemsep0pt
\item \label{enm:X_T_target} $\cp_{\Time}\sim \target,\qquad \cp_0=0$,
\item \label{enm:bridge} $\cp_t|\cp_{\Time}\sim \flow_{t|\Time}(\cdot|\cp_{\Time})=\B_{\cp_{\Time},\frac{t}{\Time}}$,
\item \label{enm:Tweedie} $\frac{\den(t,\cp_t)-\cp_t}{\Time-t}=\ints(t,\cp_t)$.
\end{enumerate}
\end{proposition}
In light of Proposition \ref{prop:key_properties} we can now combine Algorithm \ref{alg:train} and Algorithm \ref{alg:sample} to sample from the data distribution.  Algorithm \ref{alg:train} trains the denoiser $\den$ using Binomial thinning \eqref{enm:bridge}, and given the trained denoiser $\den$, Algorithm \ref{alg:sample} leverages the discrete Tweedie's formula \eqref{enm:Tweedie} to get samples from the data distribution $\target$ \eqref{enm:X_T_target}. 

The next result shows that once the intensity $\ints$ is estimated we automatically get  exact likelihood estimates. Alternatively, the identity below can be used to train the denoiser via maximum likelihood. 
\begin{proposition}[Likelihood estimation]
\label{prop:loglike}
Let $\ints$ be the intensity defined by \eqref{eq:ints_def}-\eqref{eq:ints_j_def}. Then,
\begin{equation}
\label{eq:loglike}
-\log \target(\x)=\int_0^{\Time}\E_{y\sim\B_{\x,t}} \left[\beg\left(\frac{\x-y}{\Time-t},\ints(t,y)\right)\right]\diff t,
\end{equation}
where, for $a\in \N^{\Dim}$,
\begin{equation}
\label{eq:cvx_loglike}
\cvx(a)=\sum_{\idx=1}^{\Dim}a^{\idx}\log a^{\idx},\quad \nabla \cvx(a):=(1+\log a^1,\ldots, 1+\log a^{\Dim}),
\end{equation}
and, for $a,b\in \N^{\Dim}$,
\begin{equation}
\label{eq:beg_div_def}
\begin{split}
\beg(a,b)&:=\cvx(a)-\cvx(b)-\nabla \cvx(b)\cdot(a-b)\\
&=\sum_{\idx=1}^{\Dim}\left[a^{\idx}\log a^{\idx} - a^{\idx} \log b^{\idx}-a^{\idx}+b^{\idx}\right].
\end{split}
\end{equation}
\end{proposition}
\subsection{The many facets of the Poisson-F\"ollmer process}
\label{subsec:prop_PF}
The previous sections described the Poisson-F\"ollmer process by explicitly defining its rate matrix. In this section we provide equivalent descriptions of the process which, while interesting theoretically, are not needed for the implementation of our Binomial flow discrete diffusion model. 

The first description goes through the notion of a \Def{controlled Poisson point process}. This description is made rigorous in \cref{sec:PF}, and here we will provide the rough idea which is portrayed in \cref{fig:X}. We begin with $\Dim$ independent Poisson point processes in $[0,\Time]\times [0,\ub]$ (for large enough constant $\ub$), and $\Dim$ curves $(\ints^{\idx}_t)_{t\in [0,\Time]}$ for $\idx=1,\ldots,\Dim$. The $\idx$th coordinate $(\cp_t^{\idx})_{t\in [0,T]}$ of $(\cp_t)_{t\in [0,T]}$ is a non-decreasing process in $\N$, starting at $\cp_0^{\idx}=0$, whose value at time $t$ is equal to the number of points from the $\idx$th Poisson point process that fall in the region under the curves $\ints^{\idx}$ by time $t$. Each of the curve $\ints^{\idx}$ is chosen stochastically, depending on the values of all the coordinates of $(\cp_t)_{t\in [0,T]}$. These curves can be chosen to guarantee that at time $\Time$ the process follows the data distribution, $\cp_{\Time}\sim\target$. 
\begin{figure}
\begin{centering}
\begin{tikzpicture}[scale=1]
\draw[->] (0,0) -- (5,0);
  \draw[->] (0,0) -- (0,5);
  \draw[dashed] (0,4.8) -- (5,4.8);
  \draw (1.4,0.2) -- (1.4,-0.2);
\draw [black] plot [smooth, tension=1] coordinates { (0,.5) (1,4) (3.5,2) (4.2,4.5) (5,3.8)};

\draw [black] plot [smooth, tension=1] coordinates { (0,.5) (1.5,3) (2.9,3.3) (3.5,3) (3.8, 2.8) (4, 2.2) (5,2.4)};

\node at (-0.4,4.8) {$\ub$};
\node at (5,-0.3) {$\Time$};
\node at (1.4,-0.4) {$0.5$};
\node at (6,3.8) {$(\ints_t^1)_{t\in [0,\Time]}$};
\node at (6,2.4) {$(\ints_t^2)_{t\in [0,\Time]}$};
\draw[red,] (0.2,4) circle (.3ex);
\draw[red,fill=red] (1,1) circle (.3ex);
\draw[red,fill=red] (2,2) circle (.3ex);
\draw[red,] (2.8,3) circle (.3ex);
\draw[red,fill=red] (3.7,1) circle (.3ex);
\draw[red,] (3.8,3.3) circle (.3ex);
\draw[red,fill=red] (4.5,4) circle (.3ex);

\draw[blue,fill=blue] (0.3,0.2) circle (.3ex);
\draw[blue,fill=blue] (0.8,0.8) circle (.3ex);
\draw[blue,] (2.4,4) circle (.3ex);
\draw[blue, fill] (3,2.8) circle (.3ex);
\draw[blue, ] (4.8,3) circle (.3ex);

\end{tikzpicture}
\caption{The points in $[0,\Time]\times [0,\ub]$ are generated according to $\Dim=2$ independent standard Poisson process ({\color{red}red} and {\color{blue}blue});  {\color{red} 7 red} points and {\color{blue} 5 blue}  points. At time $t\in [0,\Time]$ the value of the process {\color{red}$\cp_t^1$} (res. {\color{blue}$\cp_t^2$}) is equal to the number of points under the curve $\ints^1$ (res.  $\ints^2$) which are denoted as {\color{red}red}  (res. {\color{blue}blue} ) \textbf{filled} circles. In the figure {\color{red}$\cp_{t=0.5}^1=1$} and {\color{red}$\cp_{t=\Time}^1=4$} (res. {\color{blue}$\cp_{t=0.5}^2=2$} and {\color{blue}$\cp_{t=\Time}^2=3$}).}
\label{fig:X}
\end{centering}
\end{figure}
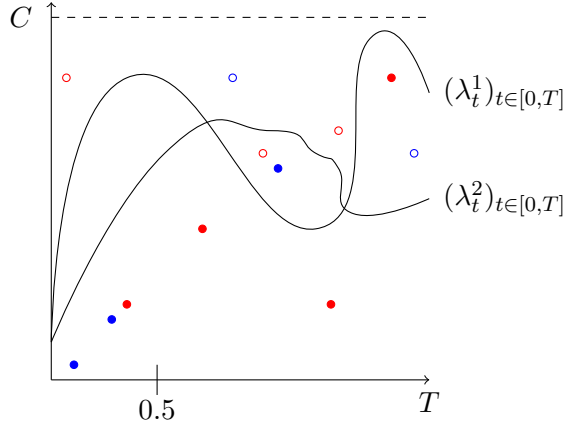 

The second description is that of a \Def{Schr\"odinger bridge} or \Def{Doob $\h$-transform}. In \cref{sec:schrondinger} we show that the Poisson-F\"ollmer process is a special case of a \Def{Schr\"odinger bridge}: Of all processes starting at 0 whose distribution at time $\Time$ is the data distribution $\target$, the Poisson-F\"ollmer process is the closest to a $\Dim$-dimensional Poisson process. Alternatively, it can be described as a \Def{Doob $\h$-transform}, where a standard Poisson process in $\N^{\Dim}$ is conditioned to be distributed like the data distribution $\target$ at time $\Time$ (rather than $\pos_{\Time}$); see 
\cite{MR3121631}. The connection between the Poisson-F\"ollmer process and the Doob $\h$-transform goes through
\begin{equation}
\label{eq:stoch_int}
\ints^{\idx}(t,\x)=\frac{\h(t,\x+\basisi)}{\h(t,\x)},\qquad \idx=1,\ldots, \Dim,
\end{equation}
where the $\h$-transform is given by $\h(t,\x)=\semi_{\Time-t}\reltarget(\x)$.

The third description is that of \Def{time reversal}. Define the \Def{forward process} $\vec{Z}_t:=\cp_{\Time-t}$ to be a \Def{Poisson bridge} which starts at $\vec{Z}_0\sim\target$ and terminates at $\vec{Z}_{\Time}=0$. Then, the Poisson-F\"ollmer process  $\cp_t=:\cev{Z_t}$ is the \Def{reverse process} obtained by the time-reversal of $\vec{Z}_t$. Hence, from the perspective of the time-reversal approach in discrete diffusions \cite{lou2024discrete}, our goal should be learn the discrete score $\left[\frac{q_t(y)}{q_t(\x)}\right]_{y\neq \x}$, where $q_t$ is the distribution of $\vec{Z}_t$. In  \cref{sec:time_rev} we show that this task is equivalent to the task of learning the intensity $\ints$. 
\begin{remark}[Gaussian analogues: The F\"ollmer process]
The Poisson-F\"ollmer process has a continuous analogue, the \Def{F\"ollmer process}, which shares many of its properties, as we explain in \cref{sec:gaussian}. As the Poisson-F\"ollmer process, the F\"ollmer process starts at a deterministic point, in contrast to most diffusion models which begin at a Gaussian distribution.  
\end{remark}

\section{Numerical experiments}
\label{sec:exp}
In this section we present our results from training a denoiser given samples from a discrete target distribution $\target$, to generate new samples from $\target$.  
\cref{sec:synthetic} presents the results on synthetic datasets, while \cref{sec:CIFAR} evaluates our method on the CIFAR-10 dataset. In all of our experiments we learn the denoiser $\den_{\theta}$ 
by minimizing the weighted squared error loss 
\begin{equation}
\label{eq:denoiser_obj_quad_exp}
\min_{\theta}\int_0^{\Time}\int_{\Space} w(t) \left|\x_{\Time}-\den_{\theta}(t,\x_t)\right|^2 \flow_{t|\Time}(\x_t|\x_{\Time})\target(\x_{\Time})\diff t
\end{equation}
over an appropriate class of neural networks. 
After learning the denoiser, we define the rate function as  
$\ints_{\theta}(t,\x_t) = \frac{\den_{\theta}(\x_t,t) - \x_t}{\Time - t}.$
We use the rate to generate new samples using \cref{alg:sample} with either an Euler or $\tau-$leaping method as in Algorithms~\ref{alg:euler} or~\ref{alg:tau}, respectively. In the following experiments we choose $T = 1$.

\begin{algorithm}[tb]
  \caption{Euler sampler}
  \label{alg:euler}
  \begin{algorithmic}
    \STATE \textbf{Input:} Initial state $\x_0\in\N^{\Dim}$, rate $\ints:[0,\Time]\times \N^{\Dim}\to \R^{\Dim}$, time step $\Delta t$
    \STATE \textbf{Initialize:} $t \gets 0$, $\x \gets \x_0$
    \WHILE{$t < \Time$}
        \STATE Evaluate rates: $\lambda(x,t)$
        \STATE Build rates:\,$R^{\idx}(\x^{\idx}\!+\!1)\!\gets\!\ints^{\idx}{(x,t)}$,\,$R^{\idx}(\x^{\idx})\!\gets\!-\ints^{\idx}{(x,t)}$, $R^{\idx}(y^{\idx}) \gets 0$ for all $y^{\idx}\in\N\backslash\{\x^{\idx},\x^{\idx}+1\}$
        \STATE Define probabilities: $P^{\idx}(y^{\idx}) \gets \delta_{\x^{\idx}}(y^{\idx}) + \Delta t \, R^{\idx}(y^{\idx})$ 
        \STATE Clip probabilities so that $P^i(y^i) \geq 0$ and re-normalize
        \STATE Sample new state: $\x^{\idx} \sim \texttt{Cat}\!\left(P^{\idx}\right)$ for each $i$
        \STATE Update time: $t \gets t + \Delta t$   
    \ENDWHILE
  \end{algorithmic}
\end{algorithm}

\begin{algorithm}[tb]
  \caption{$\tau$-leaping sampler}
  \label{alg:tau}
  \begin{algorithmic}
      \STATE \textbf{Input:} Initial state $\x_0\in\N^{\Dim}$, rate $\ints:[0,\Time]\times \N^{\Dim}\to \R^{\Dim}$, time step $\Delta t$
      \STATE \textbf{Initialize:} $t \gets 0$, $x \gets x_0$
      \WHILE{$t < T$}
          \STATE Evaluate intensities: $\lambda(x,t)$
          \STATE Sample increments: $y \sim \Poisson_{\lambda(x,t) \Delta t}$
          \STATE Update state: $x \gets x + y$
          \STATE Update time: $t \gets t + \Delta t$
      \ENDWHILE
  \end{algorithmic}
\end{algorithm}

\subsection{Synthetic data} \label{sec:synthetic}

In this section we apply our approach to various $\Dim=1$-dimensional synthetic distributions for non-negative ordinal data that are considered in~\cite{DBLP:journals/corr/abs-2505-05082}. These distributions are selected to display characteristics such as bimodality and heavy tails. The definition of each distribution and details on the model architecture and training of the denoiser for this experiment can found in \cref{app:additional_numerics}. 

We learn the model parameters by minimizing the loss~\eqref{eq:denoiser_obj_quad_exp} with uniform time-sampling and  weight $w(t) = (1 - t)^{-0.5},$ which encourages the model to more accurately approximate the true denoiser near the data distribution at $t = 1$. We learn the denoiser using 50,000 i.i.d.\ samples $x_1$ from the data distribution and generate $10,000$ i.i.d.\ samples from the learned model using an Euler sampler (Algorithm~\ref{alg:euler}) with 1000 time-steps. 

True samples and generated samples from the predicted distributions are plotted in Figure~\ref{fig:synthetic_distributions}, showing that Binomial flows closely captures both bimodal features, heavy tails and distributions with varying support. 

For each distribution, we also evaluate the performance quantitatively. We use the likelihood identity in~\eqref{eq:loglike} to evaluate the negative log-likelihood (NLL) of the true samples under the predicted distribution. We use a Monte Carlo estimator with $1000$ samples to evaluate the NLL for each data sample $x_1 \sim \mu$ and compute the average NLL over $10,000$ data samples. Table~\ref{tab:performance_comparison_NLL} reports the average and standard error of the NLL over five training instances. For most examples, we observe close agreement to the true NLL evaluated by using the exact evaluations of the target distribution. Table~\ref{tab:performance_comparison} in \cref{app:additional_numerics} also compares the $W_1$ distance between the generated and true data samples to the results presented in~\citep{DBLP:journals/corr/abs-2505-05082}. For the ZIP, Zipf and Yule-Simon distributions we observe statistically significant closer results to the target with Binomial flows.

\begin{figure}[!ht]
    \centering
    \hspace{0.5cm} \textbf{Poisson Mixture} \hspace{0.9cm} \textbf{Zero-Inflated Poisson} \\
    \includegraphics[width=0.45\linewidth]{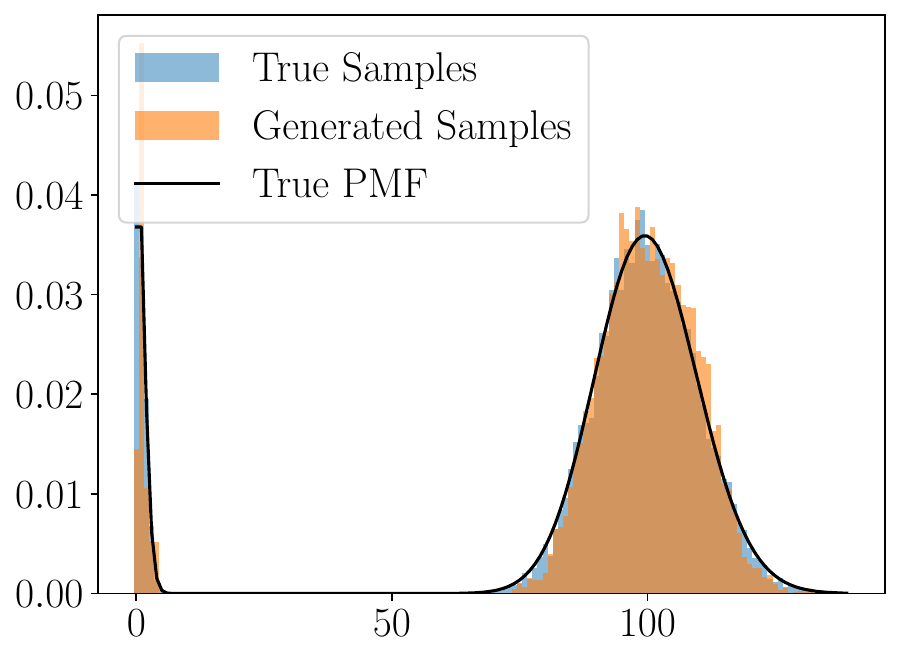}
    \includegraphics[width=0.45\linewidth]{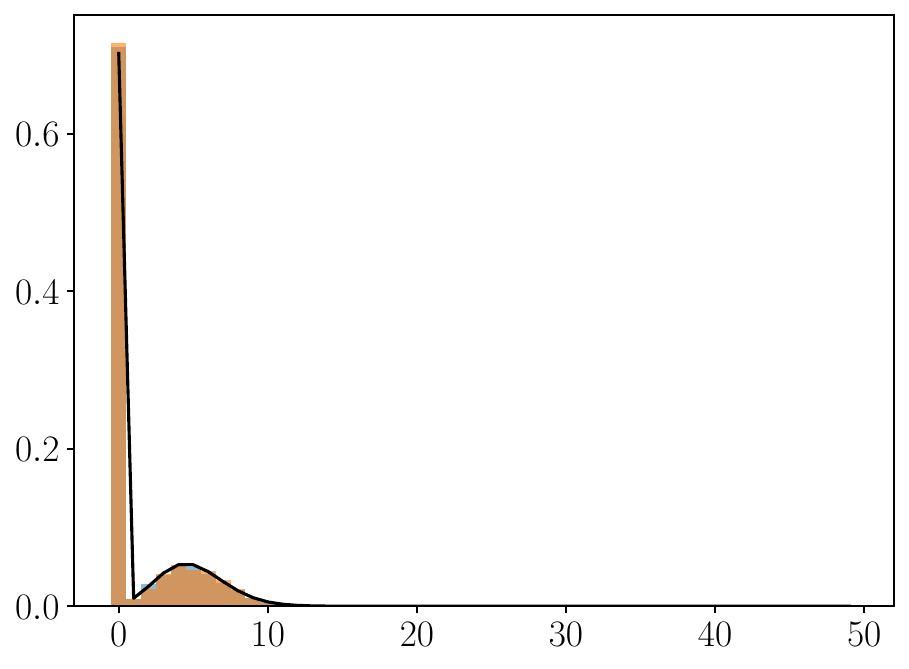}\\
    \hspace{0.6cm} \textbf{NegBinomial Mixture} \hspace{0.2cm} \textbf{Beta-Negative-Binomial} \\
    \includegraphics[width=0.45\linewidth]{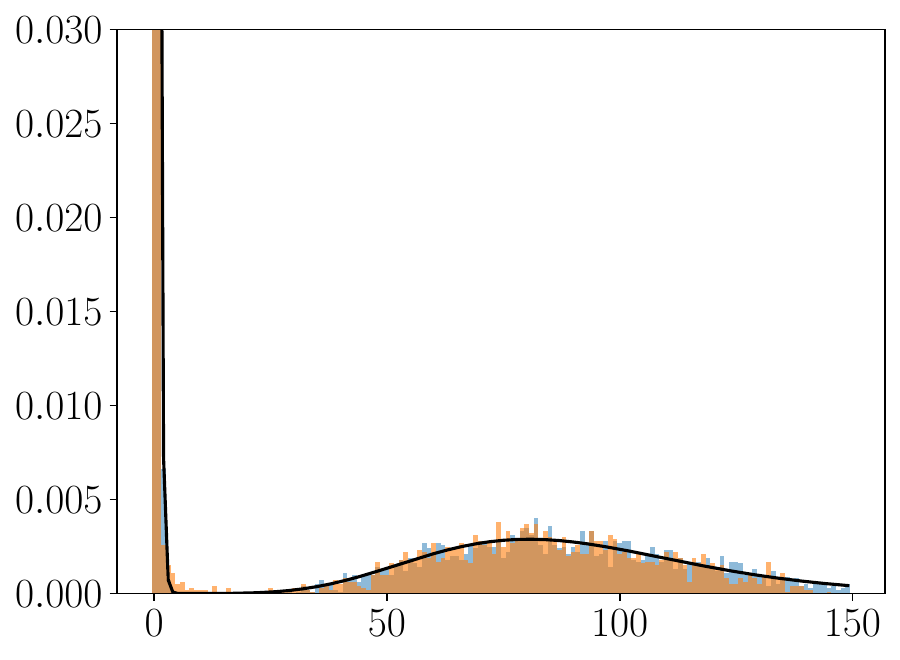}
    \includegraphics[width=0.45\linewidth]{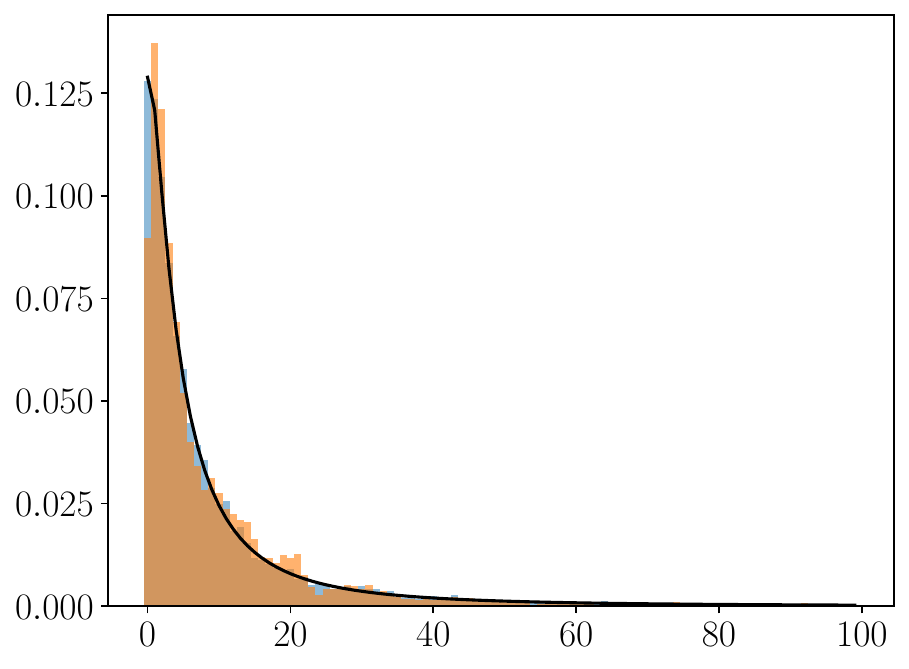}
    \caption{Target dataset, generated samples using Binomial Flows and the true target PMF from four synthetic datasets in Section~\ref{sec:synthetic} \label{fig:synthetic_distributions}}
\end{figure}

\begin{table}[!ht]
\centering
\caption{NLL evaluations for the target samples under the predicted and true distributions, showing close agreement.}
\label{tab:performance_comparison_NLL}
\begin{tabular}{lccc}
\midrule
\textbf{Problem} 
& \textbf{Binomial Flows} & \textbf{True NLL} \\  \\
\midrule
Poisson & $2.28 \pm 0.06$ & $2.21 \pm 10^{-4}$ \\ 
Poisson Mixture & $4.96 \pm 0.80$ & $3.80 \pm 10^{-4}$ \\
ZIP & $1.28 \pm 0.02$ & $1.25 \pm 10^{-4}$ \\ 
NBM & $5.34 \pm 2.21$ & $1.70 \pm 10^{-4}$ \\
BNB & $4.03 \pm 0.34$ & $3.25 \pm 10^{-4}$ \\ 
Zipf & $2.19 \pm 0.04$ & $1.96 \pm 10^{-4}$ \\ 
Yule-Simon & $1.33 \pm 0.07$ & $1.22 \pm 0.01$ \\ 
\end{tabular}
\end{table}

\subsection{CIFAR-10} \label{sec:CIFAR}

We evaluate our method on the CIFAR-10 dataset using the EDM framework~\citep{karras2022elucidating}.
We adopt the DDPM++ architecture~\citep{song2021scorebased} for the denoiser, data pipeline, training loop, and FID computation from the EDM codebase without modification. The following subsections provide some information on our training and sampling for CIFAR-10, and we refer to \cref{sec:experiments} for more information.

\subsubsection{Preconditioning and Training}
\label{subsubsec:precond_train}
We adapt EDM's preconditioning strategy to the Binomial flow setting. We emphasize that this is possible due to our ability to work with the quadratic loss \eqref{eq:denoiser_obj_quad_exp}. We parametrize the denoiser as
\begin{equation}
\label{eq:mf}
\den_{\theta}(\x_t,t)=c_{\textnormal{skip}}(t)\x_t+c_{\textnormal{out}}(t)F_{\theta}(c_{\text{in}}(t) x_t + s_{\text{in}},t),
\end{equation}
with a neural network $F_{\theta}$, where the parameters $c_{\textnormal{skip}}(t), c_{\textnormal{out}}(t), c_{\text{in}}(t), s_{\text{in}}$ are chosen in a way which stabilizes training.  We provide more details in \cref{subsec:preconditioning_appen} and summarize the parameters in \cref{tab:pf-edm}.

\subsubsection{Time Parameterization}
\label{subsubsec:time_repa}

We parameterize time using the noise level $\sigma = -\log(t + \varepsilon_\text{noise})$ for $\varepsilon_\text{noise} = 10^{-5}$, which is provided to the neural network as the time argument. 
This logarithmic parameterization is crucial for the Binomial flow, as it provides appropriate resolution near $t \approx 0$ where the signal-to-noise ratio changes rapidly.
\cref{fig:denoiser} illustrates the denoiser's behavior at $t=0.001$. 
Even though $x_t$ appears nearly indistinguishable from noise and $t \approx 0$, the denoiser already reconstructs the overall long-range structure of the image.
This should be compared with the case $t = 0$, where $\E[\cp_1|\cp_0] = \E[\cp_1]$.
Additional denoiser plots are shown in \cref{subsec:add_fig}.

\cref{fig:denoiser_evo} demonstrates that uniform discretization in noise space $\sigma$ yields a more natural progression of the generative process compared to uniform discretization in time $t$.
\begin{figure}
\centering
\includegraphics[width=0.9\linewidth]{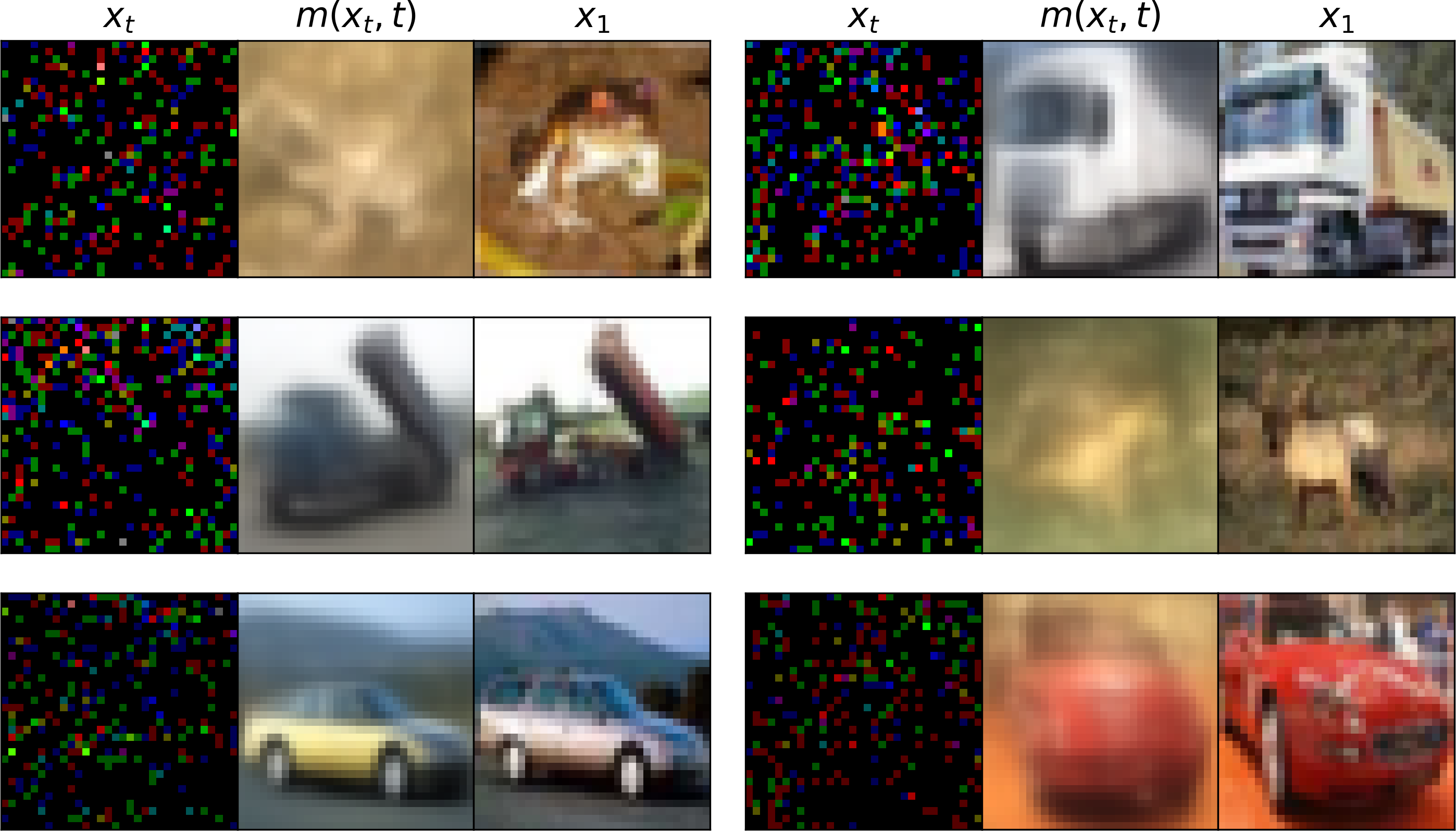}
\caption{Denoising at early time $t=0.001$. Left column: noisy observations $x_t \sim \text{Binomial}_{x_1, t}$ (normalized to fill the dynamic range). Middle column: denoiser predictions $m(x_t, t) \approx \mathbb{E}[X_1|X_t=x_t]$. Right column: clean images $x_1$ from CIFAR-10. Two groups, six samples.}\label{fig:denoiser}
\end{figure}

\begin{figure}[!htb]
\centering
\includegraphics[width=0.98\linewidth]{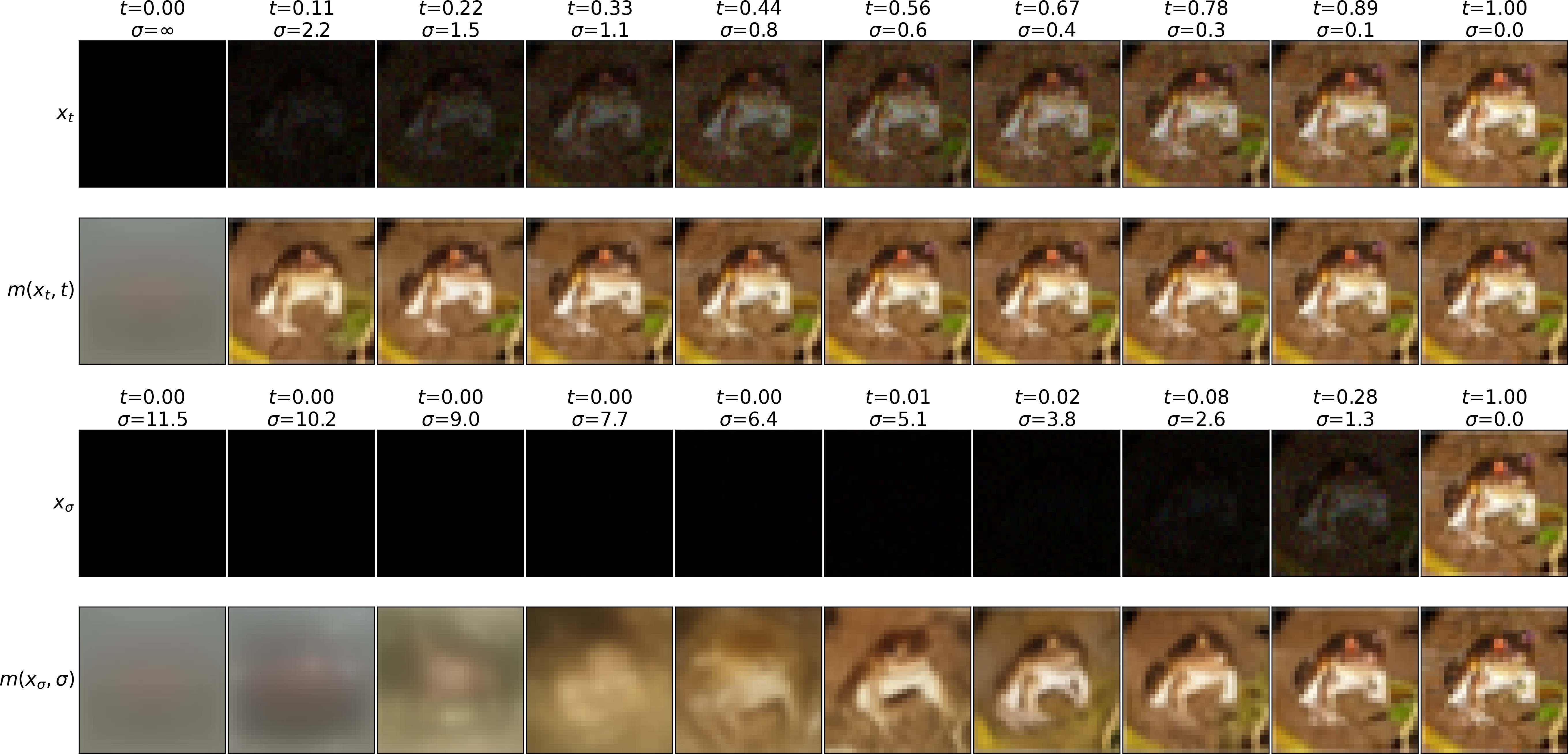}
\caption{Evolution of denoising across time. Comparison of uniform discretization in time $t$ (top) versus noise level $\sigma$ (bottom). The logarithmic parameterization provides better coverage of the denoising trajectory with a focus on the early stages. All images not normalized.}\label{fig:denoiser_evo}
\end{figure}

\subsubsection{Noise Sampling}
\label{subsubsec:noise_sample}
To design an effective noise sampling strategy for training, we analyze where the learned denoiser improves over a simple baseline.
We first derive the optimal affine baseline $\mathrm{b}(t) := b_{\text{skip}}(t) \cp_t + b_{\text{out}}(t) \mu_{\text{data}}$ that minimizes the expected squared error $\loss_t(\mathrm{b}) = \E[|\cp_1 - \mathrm{b}(t)|^2]$ at each time $t$; see Claim \ref{cl:opt_b_bin}.
We plot the time evolution of all scaling coefficients in \cref{subsec:add_fig}.

Given a trained denoiser $\den$ we compute the improvement $\loss_\sigma(\mathrm{b}) - \loss_\sigma(\den)$ as a function of the noise level $\sigma$.
This quantity measures how much the neural network improves over the baseline at each noise level and thus indicates which noise levels contribute most to learning.
As shown in \cref{fig:time}, the normalized improvement curve is well-approximated by a Gaussian distribution in $\sigma$-space.
This observation motivates sampling noise levels from a truncated Gaussian $\sigma \sim \mathcal{N}(\mu_\sigma,\gamma^2_\sigma)\mathbf{1}_{[0,-\log(\varepsilon_\text{noise})]}(\sigma)$ during training, with $\mu_\sigma$ and $\gamma_\sigma$ chosen to match the empirical improvement curve.
\cref{fig:time} compares this distribution to uniform sampling in time ($t \sim
\mathcal{U}[0,1]$), showing that the Gaussian strategy allocates more training budget to
noise levels where improvement is achievable.

\subsubsection{Results}

\begin{figure}[!htb]
\centering
\includegraphics[width=0.9\linewidth]{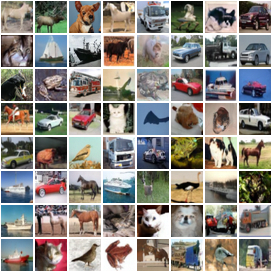}
\caption{Unconditional CIFAR-10 samples generated using the Poisson-Föllmer process with EDM preconditioning.}\label{fig:cifar10}
\end{figure}

\cref{fig:cifar10} shows unconditional samples from our model, demonstrating diverse and
realistic image generation.
We generate samples using the $\tau$-leaping sampler (\cref{alg:tau}) with 1024
uniformly-spaced time steps.
While this number of function evaluations (NFE) is higher than state-of-the-art continuous
diffusion models (e.g., EDM achieves lower FID with only 36 NFE), it is consistent with
discrete diffusion methods~\citep{gat2024discrete,chao2025mdmprime} and earlier continuous
models~\citep{song2021scorebased}.
The denoising trajectories in \cref{fig:denoiser_evo} suggest room for improved time
discretization schedules, which we leave to future work.
\cref{tab:cifar10_fid} summarizes quantitative comparisons against both continuous and
discrete diffusion baselines.

\robustify\bfseries

\begin{table}[t]
\centering
\caption{CIFAR-10 unconditional generation results. FID (lower is better) is based on 50K generated samples.
Methods are grouped by data representation: continuous (top) and discrete (bottom).}
\label{tab:cifar10_fid}
\begin{tabular}{l S[text-series-to-math, table-format=2.2]}
\toprule
Method & {FID $\downarrow$} \\
\midrule
\multicolumn{2}{l}{\textit{Continuous data models}} \\
EDM \citep{karras2022elucidating} & 1.97 \\
ScoreSDE \citep{song2021scorebased} & 2.20 \\
IDDPM \citep{pmlr-v139-nichol21a} & 2.90 \\
DDPM \citep{ho2020denoising} & 3.17 \\
\midrule
\multicolumn{2}{l}{\textit{Hybrid data models}} \\
CADD\textsuperscript{$\dagger$} \citep{zheng2025cadd} & 2.88 \\
\midrule
\multicolumn{2}{l}{\textit{Discrete (ordered) data models}} \\
\textbf{Binomial flow (Ours)} & \bfseries 2.94 \\
Blackout \citep{santos2023blackout} & 4.58 \\
LTJ \citep{chen2023learning} & 4.80 \\
ItDPDM \citep{DBLP:journals/corr/abs-2505-05082} & 4.84 \\
\midrule
\multicolumn{2}{l}{\textit{Discrete (categorical) data models}} \\
MDM-Prime \citep{chao2025mdmprime} & 3.26 \\
Discrete FM \citep{gat2024discrete} & 3.63 \\
$\tau$-LDR-10 \citep{DBLP:journals/corr/abs-2205-14987} & 3.74 \\
MDM \citep{chao2025mdmprime} & 4.66 \\
MDM-Mixture \citep{chao2025mdmprime} & 4.80 \\
D3PM (Gauss.) \citep{austin2021structured} & 7.34 \\
CTDD-DG \citep{nisonoff2025unlocking} & 7.86 \\
\bottomrule
\end{tabular}

\vspace{1mm}
{\footnotesize
\textit{Note\:}
CADD\textsuperscript{$\dagger$} is an hybrid discrete diffusion augmented with a paired continuous latent diffusion.}
\vspace{-0.2cm}
\end{table}

Our Binomial Flow approach achieves an FID of 2.94, demonstrating highly competitive performance. 
Notably, while continuous models operating on real-valued data achieve slightly better scores, our method outperforms the other discrete diffusion baselines.

\section{Related work}
\label{sec:related}

In the continuous setting, diffusion models were gradually developed over the recent decade  \cite{ho2020denoising, song2021scorebased}. Later, a flow-matching perspective was developed in the concurrent works \cite{lipman2023flow, albergo2023building, liu2023flow}. Discrete diffusions in discrete time were systematically considered in \cite{austin2021structured}, though their origin can be found in \cite{hoogeboom2021argmax} and \cite{song2021denoising}. The first paper to introduce discrete diffusions in continuous time as Continuous Time Markov Chains is \cite{DBLP:journals/corr/abs-2205-14987}. Mirroring the historical developments of diffusion models in continuous spaces, the early versions of discrete diffusions were trained using cross-entropy and evidence lower bound (ELBO) based objectives. The first paper to consider a (discrete) score based approach was \cite{lou2024discrete}, and a flow matching perspective was given in  \cite{campbell2024generative} and \cite{gat2024discrete}. 

The above approaches focus on learning the discrete score or the rate (analogous to velocity in the continuous setting), but not the denoiser (which is not a well-defined object in general discrete spaces). For binary data, e.g., $\Space=\{0,1\}^{\Dim}$, \cite{bach2025sampling, pham2025discrete} built a framework that learns both a denoiser and a discrete score. For nonnegative ordered data, as in our case, \cite{chen2023learning, DBLP:journals/corr/abs-2505-05082, montanari2023sampling} learn both a denoiser and a rate via a discrete Tweedie's formula, but a key difference is that the conditional distributions in these works are \emph{Poisson}, $\trans_{t|\Time}(\cdot|\x_{\Time})=\Poisson_{t\x_{\Time}}$, rather than \emph{Binomial} as in our work. These two types of perturbations are of different nature. For example, the Binomial perturbation recovers the clean data at $t=\Time$, while the Poisson perturbation only does so in limit $\Time\to\infty$. Additionally, the Binomial perturbation is one-sided since $\XX_t|\XX_{\Time}$ is guaranteed to be smaller than $\XX_{\Time}$, unlike the Poisson perturbation which can ``overshoot" the effective support of the data. Finally, the Binomial perturbation leads to a sampling process, namely the Poisson-F\"ollmer process, with desirable properties; cf. Section \ref{subsec:prop_PF}. We note that \cite{DBLP:journals/corr/abs-2505-05082} inspired us to find our exact likelihood formula~\eqref{eq:loglike}. Overall, it remains interesting to better understand the advantages and disadvantages of the Binomial vs. Poisson perturbations. Closest to approach is \citep{santos2023blackout} which uses Binomial thinning as we do, but lacks the discrete Tweedie's formula, and hence the connection to the Poisson-F\"ollmer process. The lack of the discrete Tweedie's formula leads \citep{santos2023blackout} to rely on a likelihood-base loss while our general framework allows for all Bregman divergences, which in turns lets us mirror the EDM framework~\citep{karras2022elucidating}.   

The origin of the Poisson-F\"ollmer process can be found in \cite{MR2841073}, but the explicit construction described in Section \ref{subsec:PF} was first given in dimension $\Dim=1$ in  \cite{MR3904394}. More generally, a Poisson processes perspective on discrete diffusion models can be found in \cite{ren2025how} and \cite{pham2025discrete}. Both the  Poisson-F\"ollmer, and its Gaussian analogue the F\"ollmer process, are special cases of Schr\"odinger bridges \cite{MR3121631}. In the context of generative modeling, the question of learning Schr\"odinger bridges in the continuous setting was taken by a number of works \cite{tzen2019theoretical, bortoli2021diffusion, JMLR:v24:23-0527, shi2023diffusion, chen2024probabilistic}, and in the discrete setting by \cite{kim2024discrete, ksenofontov2025categorical}.

\section*{Conclusions and outlook}

This work introduces Binomial flows as an analogue of Gaussian-based denoising for discrete diffusion models with non-negative ordinal data. The flows are defined by a denoiser, which can be expressed as the rate function for a Poisson-F\"ollmer process. Moreover, we show the denoiser can be learned by minimizing a mean-squared error loss given target data, and it provides exact likelihood estimation. While our framework obtains low FID values on the CIFAR-10 dataset, we observed higher sensitivity of FID to the classifier-free guidance (CFG) parameter for generation on the ImageNet dataset (with FID values around ten). In the future, it is of interest to study the interaction of our denoiser with CFG and identify optimal sampling parameters (e.g., time schedules for $\tau$-leaping) that minimize the computational cost %
of inference with our method. Finally, it will be interesting to investigate whether the martingale property of the denoiser $\E[\den(t,\cp_t)|\cp_s]=\den(s,\cp_s)$ can be used to improve traningin as in \cite{daras2023consistent}.

\bibliography{binomial}
\bibliographystyle{alpha}

\newpage
\appendix
\onecolumn

\subsection*{Content of Appendices} \cref{sec:poisson_semigroup} covers some preliminaries on discrete calculus and the Poisson semigroup.  \cref{sec:poisson_point} explains how to construct counting processes out of Poisson point processes, and \cref{sec:PF} builds on these ideas to construct the Poisson-F\"ollmer process. \cref{sec:kolmogorov} derives the forward and backward Kolmogorov equations for the Poisson-F\"ollmer process. In \cref{sec:proofs} we use this preliminary work to complete the proofs of Proposition \ref{prop:key_properties} and Proposition \ref{prop:loglike}. In \cref{sec:schrondinger} we explain the connection between the Poisson-F\"ollmer process and Schr\"odinger bridges, and in \cref{sec:time_rev} we explain the connection to time-reversal.  \cref{sec:gaussian} reviews the analogous construction of the Poisson-F\"ollmer process in the continuous setting.  \cref{app:additional_numerics} and \cref{sec:experiments} contain additional information about the experiments with synthetic and imaging data, respectively.

\section{The Poisson semigroup}
\label{sec:poisson_semigroup}

We begin in \cref{subsec:prelim} by establishing our notation and by covering some preliminaries in discrete calculus. We then define the Poisson semigroup in \cref{subsec:poisson_semigroup} and derive some of its basic properties.

\subsection{Discrete calculus}
\label{subsec:prelim}
We let $\Dim\in \N$ be the dimension and write $[\Dim]:=\{1,\ldots,\Dim\}$. For $x,y\in\Z^{\Dim}$ we write $x\le y$ if $x^{\idx}\le y^{\idx}$ for all $\idx\in [\Dim]$. We use the notation $\x=(x^{\idx},\x^{-\idx})=(\x^1,\ldots,\x^{\idx-1},\x^{\idx},\x^{\idx+1},\ldots,\x^{\Dim})$ so that $\x^{-\idx}$ stands for all the coordinates in $\x$ except for $\x^{\idx}$. We let $\{\basisi\}_{\idx=1}^{\Dim}$ be the one-hot vectors, that is, $\basisi$ is the vector of all zeros except for the $\idx$th entry which is equal to 1. The discrete partial derivative $\der_{\idx}$ of a function $\fn:\Z^{\Dim}\to \R$ is defined as, for $\idx\in [\Dim]$ and $\x\in\Z^{\Dim}$,
\begin{equation}
\label{eq:partial_der}
\der_{\idx}\fn(\x):=\fn(\x+\basisi)-\fn(\x).
\end{equation}
Analogous to the continuous setting there is an ``integration by parts" formula on $\Z^{\Dim}$. 
\begin{lemma}[Summation by parts]
Let $\fn,H:\Z^{\Dim}\to \R$. Then, for each $\idx\in[\Dim]$, 
\begin{equation}
\label{eq:sum_by_parts}
\sum_{\x\in \Z^{\Dim}}\fn(x)\der_{\idx}H(\x)=-\sum_{\x\in \Z^{\Dim}}H(\x)\der_{\idx}\fn(x-\basisi).
\end{equation}
\end{lemma}
\begin{proof}
By re-indexing,
\begin{align}
\label{eq:proof_int_by_part_1}
\begin{split}
\sum_{\x\in \Z^{\Dim}}\fn(x)H(\x+\basisi)&=\sum_{\x^{-\idx}\in \Z^{\Dim-1}}\sum_{\x^{\idx}\in\Z}\fn(\x^{\idx},\x^{-\idx})H(\x^{\idx}+1,\x^{-\idx})=\sum_{\x^{-\idx}\in \Z^{\Dim-1}}\sum_{\x^{\idx}\in\Z}\fn(\x^{\idx}-1,\x^{-\idx})H(\x^{\idx},\x^{-\idx})\\
&=\sum_{\x\in \Z^{\Dim}}\fn(x-\basisi)H(\x),
\end{split}
\end{align}
so
\begin{align}
\label{eq:proof_int_by_part_2}
\sum_{\x\in \Z^{\Dim}}\fn(x)\der_{\idx}H(\x)&=\sum_{\x\in \Z^{\Dim}}\fn(x)[H(\x+\basisi)-H(\x)]\overset{\eqref{eq:proof_int_by_part_1}}{=}\sum_{\x\in \Z^{\Dim}}\fn(x-\basisi)H(\x)-\sum_{\x\in \Z^{\Dim}}\fn(x)H(\x)\\
&=-\sum_{\x\in \Z^{\Dim}}H(\x)\der_{\idx}\fn(x-\basisi).
\end{align}
\end{proof}
\subsection{The Poisson semigroup}
\label{subsec:poisson_semigroup}
We begin  by recalling the definition of the \Def{Poisson distribution} $\pos_t$  with parameter $t\ge 0$, and its multi-dimensional extension,
\begin{equation}
\label{eq:appen_pos_def}
\pos_t(k)=e^{-t}\frac{t^k}{k!},\qquad k\in \N,\qquad \pos_t(x)=\prod_{\idx=1}^{\Dim}\pos_t(\x^{\idx}),\qquad x=(\x^1,\ldots,\x^{\Dim})\in \N^{\Dim}. 
\end{equation}
The Poisson distribution $\pos_t$  satisfies the analogue of the heat equation for the Gaussian distribution. 
\begin{lemma}[Poisson heat equation]
Let $\pos_t$ be as in \eqref{eq:appen_pos_def} and extend it $\mathbb Z^{\Dim}$ by setting $\pos_t(\x)=0$ for $x\in \mathbb Z^{\Dim}\backslash \N^{\Dim}$. Then, for $t\ge 0$ and $\x\in\N^{\Dim}$,
\begin{equation}
\label{eq:Poisson_pde}
\partial_t\pos_t(x)=-\sum_{\idx=1}^{\Dim}\der_{\idx}\pos_t(\x-\basisi).
\end{equation}
\end{lemma}
\begin{proof}
We use the convention $k!=0$ for $k<0$. Then,
\begin{align}
\label{eq:Poisson_pde_proof_1}
\begin{split}
\frac{\partial_t\pos_t(\x)}{\pos_t(\x)}&=\partial_t\log\pos_t(\x)=\partial_t\left[\sum_{\idx=1}^{\Dim}\log \pos_t(\x^{\idx})\right]=\partial_t\left[\sum_{\idx=1}^{\Dim}\log \left(\frac{e^{-t}t^{\x^{\idx}}}{\x^{\idx}!}\right)\right]=\sum_{\idx=1}^{\Dim}\partial_t\left[-t+\x^{\idx}\log t-\log(\x^{\idx}!)\right]\\
&=\sum_{\idx=1}^{\Dim}\left[-1+\frac{\x^{\idx}}{t}\right],
\end{split}
\end{align}
so
\begin{align}
\label{eq:Poisson_pde_proof_2}
\begin{split}
\partial_t\pos_t(\x) &=-\Dim \,\pos_t(\x) + \sum_{\idx=1}^{\Dim}\frac{\x^{\idx}}{t}\frac{e^{-t}t^{\x^{\idx}}}{\x^{\idx}!}\prod_{l\neq \idx}\frac{e^{-t}t^{\x^{l}}}{\x^{l}!}=-\Dim \,\pos_t(\x) + \sum_{\idx=1}^{\Dim}\frac{e^{-t}t^{\x^{\idx-1}}}{(\x^{\idx}-1)!}\prod_{l\neq \idx}\frac{e^{-t}t^{\x^{l}}}{\x^{l}!}=-\Dim \,\pos_t(\x) + \sum_{\idx=1}^{\Dim}\pos_t(\x-\basisi)\\
&=\sum_{\idx=1}^{\Dim}\left[\pos_t(\x-\basisi) -\pos_t(\x) \right]=-\sum_{\idx=1}^{\Dim}\der_{\idx}\pos_t(\x-\basisi).
\end{split}
\end{align}
\end{proof}
In analogy with the heat semigroup defined by a convolution with a Gaussian distribution, one can define the \Def{Poisson semigroup} defined by a discrete convolution with a Poisson distribution. Let $\fn:\N^{\Dim}\to \R$ be such that $\sum_{\x\in \N^{\Dim}}|\fn(\x)|\pos_t(x)<+\infty$. The Poisson semigroup $(\semi_t)_{t\ge 0}$ acts on $\fn$ according to
\begin{equation}
\label{eq:appen_Poss_semi_def}
\semi_t\fn(x):=\sum_{y\in \N^{\Dim}}\fn(\x+y)\pos_t(y).
\end{equation}
The next result shows that the Poisson semigroup satisfies the Poisson heat equation. 
\begin{lemma}[Poisson semigroup equation]
\label{lem:poisson_semi_pde}
Let $\fn:\N^{\Dim}\to \R$ be such that $\sum_{\x\in \N^{\Dim}}|\fn(\x)|\pos_t(x)<+\infty$. Then,
\begin{equation}
\label{eq:Poisson_semi_pde}
\partial_t\semi_t\fn(x)=\sum_{\idx=1}^{\Dim}\der_{\idx}\semi_t\fn(x).
\end{equation}
\end{lemma}
\begin{proof}
\begin{align}
\label{eq:Poisson_semi_pde_proof}
\partial_t\semi_t\fn(x)&=\partial_t\sum_{y\in\N^{\Dim}}\fn(x+y)\pos_t(y)\overset{\eqref{eq:Poisson_pde}}{=}-\sum_{y\in\N^{\Dim}}\fn(x+y)\sum_{\idx=1}^{\Dim}\der_{\idx}\pos_t(y-\basisi)\\
&\overset{\eqref{eq:sum_by_parts}}{=}\sum_{\idx=1}^{\Dim}\sum_{y\in\N^{\Dim}}\der_{\idx}\fn(x+y)\pos_t(y)=\sum_{\idx=1}^{\Dim}\der_{\idx}\semi_t\fn(x).
\end{align}
\end{proof}
\section{Controlled Poisson point processes}
\label{sec:poisson_point}
We begin this section by defining a particular case of \Def{Poisson point processes}. We then explain how such processes can be controlled via \Def{ stochastic intensities} to construct general counting processes on $\N^{\Dim}$. We conclude by reviewing some preliminaries of stochastic calculus in Poisson space.  

Let $\Time>0$, $\ub>0$, and let $\Sett:=[0,\Time]\times [0,\ub]$. Let $\Setsig$ be the sigma-algebra generated by the Borel sets of $\Sett$ endowed with the product topology, and let $\Leb$ be the Lebesgue measure on $\Setsig$. Define the \Def{Poisson space} $(\Omega^{\idx},\mathcal F^{\idx},\PP^{\idx})$ over $(\Sett,\Setsig,\Leb)$ as the space of atomic measures, each of which has countably many atoms,
\begin{equation}
\label{eq:Omega_def}
\Omega^{\idx}:=\left\{\omega^{\idx}:\omega^{\idx}=\sum_j\delta_{(t_j,z_j)},\quad (t_j,z_j)\in \Set\quad\textnormal{ (at most countable)} \right\},
\end{equation}
the sigma-algebra $\mathcal F^{\idx}$ is  generated by the set of functions $\omega^{\idx}\mapsto\omega^{\idx}(A)$, for any fixed $ A\in\Setsig$,
\begin{equation}
\label{eq:sig_alg_poisson_def}
\mathcal F^{\idx}:=\sigma\left(\Omega^{\idx}\ni\omega^{\idx}\mapsto\omega^{\idx}(A): A\in\Setsig\right),
\end{equation}
and the  probability measure $\PP^{\idx}$ is defined by the requirements
\begin{equation}
\label{eq:prob_poisson_def}
\begin{split}
&\forall~ A\in\Setsig,~\forall~k\in\N,\quad \PP^{\idx}[\{\omega^{\idx}\in\Omega^{\idx}:\omega^{\idx}(A)=k\}]=\pos_{\Leb(A)}(k),\\
&\forall~k\in \N,\quad \omega^{\idx}(A_1),\ldots,\omega^{\idx}(A_k)\textnormal{ are $\PP^{\idx}$-independent if $A_1,\ldots,A_k\in\Setsig$ are disjoint}.\\
\end{split}
\end{equation}
In words, under $\PP^{\idx}$, the number of points that fall in disjoint regions are independent, and the number of points that fall in  a region $A$ is distributed like a Poisson whose parameter is the Lebesgue volume of $A$. We now define the space $(\Omega,\mathcal F,\PP):=\bigotimes_{\idx=1}^{\Dim}(\Omega^{\idx},\mathcal F^{\idx},\PP^{\idx})$ to be the product space of the Poisson spaces $\{(\Omega^{\idx},\mathcal F^{\idx},\PP^{\idx})\}_{\idx=1}^{\Dim}$. Note that, by construction, 
\begin{equation}
\label{eq:ind_poss}
 \forall~ A_1,\ldots, A_{\Dim}\in\Setsig,\quad \omega^1(A_1),\ldots,\omega^{\Dim}(A_{\Dim})\textnormal{ are $\PP$-independent}.
\end{equation}
The space $(\Omega,\mathcal F,\PP)$ is the analogue of the Wiener space with Wiener measure under which continuous functions $[0,\Time]\to\R^{\Dim}$ are standard $\Dim$-dimensional Brownian motions. 

Given $t\in [0,\Time]$ let $\Setsig_t$ be the sigma-algebra generated by the Borel sets of $[0,t]\times [0,\ub]$, and let
\begin{equation}
\label{eq:filt}
\mathcal F_t:=\sigma\left(\Omega\ni\omega\mapsto\omega(A): A\in\Setsig_t\right).
\end{equation}
 In the continuous setting one works with stochastic differential equations which are generated by taking a Brownian motion and adding a drift and/or multiplying by diffusion matrix. In the discrete setting this role is taken by \Def{stochastic intensities}. These are $\Dim$-dimensional stochastic processes $\ints=(\ints^1_t\ldots,\ints^{\Dim}_t):\Omega\to\R^{\Dim}$ which are \Def{predictable}, i.e., the function $[0,\Time]\times\Omega\ni(t,\omega)\mapsto \ints_t(\omega)$ is measurable with respect to the sigma-algebra $\sigma\left(\{(s,t]\times A: s\le t\le \Time, A\in \Setsig_s\}\right)$.

We now show how to generate a counting process using an intensity $\ints$. Suppose $\ints$ is an intensity such that $\PP$-almost-surely $|\ints^{\idx}_t|\le \ub$ for all $t\in [0,\Time]$ and all $\idx\in [\Dim]$. We define the counting process $(\cp_t)_{t\in [0,\Time]}$ by 
\begin{equation}
\label{eq:count_process_def}
\cp_t^{\idx}(\omega):=\omega^{\idx}\left(\{(s,z)\in \Set:s<t, z<\ints_s^{\idx}(\omega)\}\right),
\end{equation}
where $\cp_t=(\cp_t^1,\ldots,\cp_t^{\Dim})\in \N^{\Dim}$. In words, the value of $\cp_t^{\idx}$ is the number of points in $\omega^{\idx}$ in $[0,t]\times [0,\ub]$ which fall under the curve $(\ints_s^{\idx})_{s\in [0,t]}$; see Figure \ref{fig:X}. 

By construction, each of the coordinates $(\cp_t^{\idx})_{t\in [0,\Time]}$ is a process adapted to the filtration $(\mathcal F_t)_{t\in [0,\Time]}$, non-decreasing, integer-valued, and left-continuous. Further, since, with probability 1, each $\omega^{\idx}$ has only finitely many atoms in $[0,t]\times [0,\ub]$, and no two atoms lie on the same vertical line $\{t\}\times [0,\ub]$, we get that $(\cp_t^{\idx})_{t\in [0,\Time]}$ has finitely many jumps, each of size 1. Note that the coordinates of $\cp_t$ need not be independent (and indeed in our case they will generally be dependent), since the value of each $\cp_t^{\idx}$ is determined by $(\ints_s^{\idx})_{s\in [0,t]}$, and the processes $\{(\ints_t^{\idx})_{t\in [0,\Time]}\}_{\idx=1}^{\Dim}$ can be dependent. (The analogy in the continuous setting is that although the coordinates of a standard Brownian motion are independent, the coordinates of a stochastic differential equation built on this Brownian motion will generally be dependent.) However, with probability 1, no two coordinates $(\cp_t^{\idx})_{t\in [0,\Time]}$ and $(\cp_t^j)_{t\in [0,\Time]}$ will jump at the same time $t$ since there is zero probability for atoms coming from $\omega^{\idx}$ and $\omega^j$ to lie on the same vertical line $\{t\}\times [0,\ub]$.

The next result is the discrete analogue of the It\^o formula for counting processes. This is a classical result (e.g, p. 16 in \cite{MR3904394}), so we will only sketch its proof. To explain the result we first need to explain the discrete analogue of a ``stochastic integral", which is in fact much simpler in the discrete setting and can be defined pathwise. The stochastic integral $\int_s^t \cdot \diff \cp_r^{\idx}(\omega)$ is the usual Riemann-Stieltjes integral given by, for a function $H:[0,\Time]\to\R$,
\begin{equation}
\label{eq:stoch_integ_def}
\int_s^t H_r\diff \cp_r^{\idx}(\omega):=\sum_{s\le r_j\le t} H_{r_j^{\idx}},
\end{equation}
where $\{r_j^{\idx}\}_j$ are the jump times of $(\cp_r^{\idx}(\omega))_{r\in [0,\Time]}$ between $s$ and $t$. 
\begin{lemma}[Discrete  It\^o formula]
\label{lem:ito}
Let $G:[0,\Time]\times\N^{\Dim}\to\R$ be a function which is differentiable in $t$. Then, $\PP$-almost-surely,
\begin{equation}
\label{eq:Ito}
\fn(t,\cp_t)=\fn(s,\cp_s)+\int_s^t\partial_r\fn(r,\cp_r)\diff r+\int_s^t\sum_{\idx=1}^{\Dim}\der_{\idx}\fn(r,\cp_r)\diff \cp_r^{\idx}.
\end{equation}
\end{lemma}
\begin{proof}
By construction, almost-surely, $[s,\Time]\ni t\mapsto \fn(t,\cp_t)$ is a piecewise absolutely continuous function in $t$, so, almost-surely, its distributional derivative is a sum of an integrable function on $[s,\Time]$ and finitely many atoms at $s\le r_j\le t$. The integrable function part is $\int_s^t\partial_r\fn(r,\cp_r)\diff r$, while for the jump part we note that at each jump time $r_j$ the distributional derivative is is equal to $\der_{\idx}\fn(r_j,\cp_{r_j})$, where $\idx$ corresponds to the coordinate that jumped at time $r_j$. Thus, we can decompose the sum over jumps to a double sum, one over $\idx\in [\Dim]$ and one over jumps at which the $\idx$th coordinate jumped, and then use the definition \eqref{eq:stoch_integ_def} to conclude the result. 
\end{proof}
The next result tells us how to convert stochastic integrals to deterministic integrals. 
\begin{lemma}
\label{lem:stoch_to_determ}
Let $(H_t)_{t\in [0,\Time]}=(H_t^1,\ldots, H_t^{\Dim})_{t\in [0,\Time]}$ be a nonnegative $\Dim$-dimensional predictable process. Then, for any $0\le s\le t\le \Time$,
\begin{equation}
\label{eq:stoch_to_determ}
\E\left[\int_s^tH_r\cdot \diff \cp_r \right]=\E\left[\int_s^tH_r\cdot\ints_r \diff r\right].
\end{equation}
\end{lemma}
\begin{proof}
By summing over $\idx\in [\Dim]$ Equation \eqref{eq:stoch_to_determ} is an immediate consequence of Lemma 4.1 in  \cite{MR3904394}.
\end{proof} 
As a consequence of Lemma \ref{lem:ito} and Lemma \ref{lem:stoch_to_determ} we get the following result.
\begin{corollary}
\label{cor:stoch_integ}
Let $G:\N^{\Dim}\to\R_{\ge 0}$ be a nonnegative function. Then,
\begin{equation}
\label{eq:stoch_integ}
\E[\fn(\cp_t)|\cp_s=\x_s]-\fn(\x_s)=\sum_{\idx =1}^{\Dim}\int_s^t\E[\der_{\idx}\fn(\cp_r)\ints^{\idx}_r|\cp_s=\x_s]\diff r.
\end{equation}
\end{corollary}
\begin{proof}
By Lemma \ref{lem:ito},
\begin{align*}
\fn(\cp_t)=\fn(\cp_s)+\int_s^t\sum_{\idx=1}^{\Dim}\der_{\idx}\fn(\cp_r)\diff \cp_r^{\idx},
\end{align*}
so applying Lemma \ref{lem:stoch_to_determ}  conditionally on $\cp_s=\x_s$ gives 
\begin{align*}
&\E[\fn(\cp_t)|\cp_s=\x_s]=\fn(\x_s)+\sum_{\idx=1}^{\Dim}\E\left[\int_s^t\fn(\cp_r+\basisi)\diff \cp_r^{\idx}\bigg|\cp_s=\x_s\right]-\sum_{\idx=1}^{\Dim}\E\left[\int_s^t\fn(\cp_r)\diff \cp_r^{\idx}\bigg|\cp_s=\x_s\right]\\
&=\fn(\x_s)+\sum_{\idx=1}^{\Dim}\E\left[\int_s^t\fn(\cp_r+\basisi)\ints_r^{\idx}\diff r\bigg|\cp_s=\x_s\right]-\sum_{\idx=1}^{\Dim}\E\left[\int_s^t\fn(\cp_r)\ints_r^{\idx}\diff r\bigg|\cp_s=\x_s\right]\\
&=\fn(\x_s)+\sum_{\idx=1}^{\Dim}\E\left[\int_s^t\der_{\idx}\fn(\cp_r)\ints_r^{\idx}\diff r\bigg|\cp_s=\x_s\right]=\fn(\x_s)+\sum_{\idx=1}^{\Dim}\int_s^t\E[\der_{\idx}\fn(\cp_r)\ints_r^{\idx}|\cp_s=\x_s]\diff r.
\end{align*}
\end{proof}

\section{The Poisson-F\"ollmer process}
\label{sec:PF}
In this section we construct the Poisson-F\"ollmer process using a simple extension of the one-dimensional construction due to  \cite{MR3904394}. In \cref{sec:poisson_point} we saw how standard Poisson processes can be modified using stochastic intensities to yield counting processes $(\cp_t)_{t\in [0,\Time]}$ as in \cref{eq:count_process_def}. In this section we will choose the intensity $(\ints_t)_{t\in [0,\Time]}$ to depend on the process $(\cp_t)_{t\in [0,\Time]}$, which the intensity defines. This is analogous to the continuous setting where the drift and/or diffusion matrix depend on the value of the solution to the stochastic differential equation. The construction of such process goes through a fixed point argument. 

Recall that we write the data distribution as
\[
\target=\reltarget\pos_{\Time},
\]
with the assumption that $\reltarget$ is bounded from above and below by positive constants. Let
\begin{equation}
\label{eq:h_def_lem_appen}
\h(t,\x)=\semi_{\Time-t}\reltarget(x),
\end{equation}
and define the intensity  $\ints=(\ints^1,\ldots,\ints^{\Dim}):[0,\Time]\times \N^{\Dim}\to \R^{\Dim}$ by 
\begin{equation}
\label{eq:stoch_int_appen}
\ints^{\idx}(t,\x):=\frac{\h(t,\x+\basisi)}{\h(t,\x)},\qquad \idx\in [\Dim]. 
\end{equation}
Let us derive some  properties $\h$ which will be useful later.
\begin{lemma}
Let $\h:\N^{\Dim}\to \R$  be as in \cref{eq:h_def_lem_appen}. Then,
\begin{equation}
\label{eq:h_pde}
\partial_t\h(t,\x)=-\sum_{\idx=1}^{\Dim}\der_{\idx}\h(t,\x),
\end{equation}
\begin{equation}
\label{eq:h_log_pde}
 \partial_t\log \h(t,\x)=\sum_{\idx=1}^{\Dim}\left[1-\frac{\h(t,\x+\basisi)}{\h(t,\x)}\right]=\sum_{\idx=1}^{\Dim}\left[1-\ints^{\idx}(t,\x)\right],
\end{equation}
and
\begin{equation}
\label{eq:h_bd}
 \h(\Time,\x)=\reltarget(x),\qquad \h(0,0)=1.
\end{equation}
\end{lemma}
\begin{proof}
Equation in \eqref{eq:h_pde} follows immediately from \eqref{eq:Poisson_semi_pde}. Equation \eqref{eq:h_log_pde} follows from
\begin{align*}
\partial_t\log\h(t,\x)=\frac{\partial_t\h(t,\x)}{\h(t,\x)}\overset{\eqref{eq:h_pde}}{=}-\sum_{\idx=1}^{\Dim}\frac{\der_{\idx}\h(t,\x)}{\h(t,\x)}=-\sum_{\idx=1}^{\Dim}\left[\frac{\h(t,\x+\basisi)}{\h(t,\x)}-1\right]\overset{\eqref{eq:stoch_int_appen}}{=}\sum_{\idx=1}^{\Dim}\left[1-\ints^{\idx}(t,\x)\right].
\end{align*}
Equation \eqref{eq:h_bd} follows from
\[
\h(0,0)=\semi_{\Time}\reltarget(0)=\sum_{y\in\N^{\Dim}}\reltarget(y)\pos_{\Time}(y)=\sum_{y\in\N^{\Dim}}\target(y)=1.
\]
\end{proof}
We now turn to the existence of the Poisson-F\"ollmer process. 
\begin{lemma}[Existence of the Poisson-F\"ollmer process]
\label{lem:pf_exists}
There exists a counting process $(\cp_t)_{t\in [0,\Time]}$ defined via \cref{eq:count_process_def} with $\ints^{\idx}_t:=\ints^{\idx}(t,\cp_t)$ where $\ints^{\idx}(t,\x)$ is as in \cref{eq:stoch_int_appen} for $\idx\in [\Dim]$.
\end{lemma}
\begin{proof}
We follow the proof of Lemma 4.3 in  \cite{MR3904394}. To avoid confusion set $\fn^{\idx}(t,\x):=\ints^{\idx}(t,\x)=\frac{\h(t,\x+\basisi)}{\h(t,\x)}$ and let $\fn:=(\fn^1,\ldots,\fn^{\Dim})$. Given a predictable intensity $(\ints_t)_{t\in [0,\Time]}$ let $(\cp_{t,\ints})_{t\in [0,\Time]}$ denote the process defined via \eqref{eq:count_process_def}. Let $\mathcal H:(\ints_t)_{t\in [0,\Time]}\mapsto (\fn(t,\cp_{t,\ints}))_{t\in [0,\Time]}$ be a function from the set of predictable nonnegative processes to itself. The proof of the lemma will be complete if $\mathcal H$ has a fixed point. To this end let $\ints,\alpha$ be two predictable nonnegative processes. Each of the coordinates of $\fn(t,x)$ is bounded from above and below by positive constants since $\reltarget$ is bounded from above and below by positive constants, and by the definition of $\h$. Hence,
\begin{equation}
\label{eq:lip0}
\E[|\fn(t,\cp_{t,\ints})-\fn(t,\cp_{t,\alpha}) |]\le c\,\PP[\cp_{t,\ints}\neq \cp_{t,\alpha}],
\end{equation}
for some constant $c>0$, where $|\cdot|$ is the $L^1$-norm. Since $\cp_{t,\ints},\cp_{t,\alpha}$ are both $\N^{\Dim}$-valued we have 
\begin{equation}
\label{eq:prob_bd_by_exp}
\PP[\cp_{t,\ints}\neq \cp_{t,\alpha}]\le \E[|\cp_{t,\ints}-\cp_{t,\alpha}|],
\end{equation}
because when $\cp_{t,\ints}\neq \cp_{t,\alpha}$ the value of $|\cp_{t,\ints}^{\idx}- \cp_{t,\alpha}^{\idx}|$, for at least one  coordinate, is at least 1. Fix $\idx\in [\Dim]$. By Lemma \ref{lem:stoch_to_determ},
\begin{align*}
&\E[1_{\{\cp_{t,\ints}^{\idx}\ge \cp_{t,\alpha}^{\idx}\}}(\cp_{t,\ints}^{\idx}-\cp_{t,\alpha}^{\idx})]=\E\left[\int_0^t 1_{\{\cp_{t,\ints}^{\idx}\ge \cp_{t,\alpha}^{\idx}\}}\diff\cp_{r,\ints}^{\idx}-\int_0^t 1_{\{\cp_{t,\ints}^{\idx}\ge \cp_{t,\alpha}^{\idx}\}}\diff\cp_{r,\alpha}^{\idx}\right]\\
&=\E\left[\int_0^t 1_{\{\cp_{t,\ints}^{\idx}\ge \cp_{t,\alpha}^{\idx}\}}\ints_r^{\idx}\diff r-\int_0^t 1_{\{\cp_{t,\ints}^{\idx}\ge \cp_{t,\alpha}^{\idx}\}}\alpha_r^{\idx}\diff r\right]\le \E\left[\int_0^t|\ints_r^{\idx}- \alpha_r^{\idx}|\diff r\right],
\end{align*}
and similarly, 
\begin{align*}
&\E[1_{\{\cp_{t,\ints}^{\idx}<\cp_{t,\alpha}^{\idx}\}}(\cp_{t,\ints}^{\idx}-\cp_{t,\alpha}^{\idx})]\le \E\left[\int_0^t|\ints_r^{\idx}- \alpha_r^{\idx}|\diff r\right].
\end{align*}
Hence,
\begin{align*}
\E[|\cp_{t,\ints}^{\idx}-\cp_{t,\alpha}^{\idx}|]&=\E\left[1_{\{\cp_{t,\ints}^{\idx}\ge \cp_{t,\alpha}^{\idx}\}}(\cp_{t,\ints}^{\idx}-\cp_{t,\alpha}^{\idx})\right]+\E\left[1_{\{\cp_{t,\ints}^{\idx}<\cp_{t,\alpha}^{\idx}\}}(\cp_{t,\ints}^{\idx}-\cp_{t,\alpha}^{\idx})\right]\le2\E\left[\int_0^t|\ints_r^{\idx}- \alpha_r^{\idx}|\diff r\right],
\end{align*}
so it follows that 
\begin{equation}
\label{eq:lip1}
\E[|\cp_{t,\ints}-\cp_{t,\alpha}|]\le 2\,\E\left[\int_0^t|\ints_r- \alpha_r|\diff r\right].
\end{equation}
Combining \eqref{eq:lip0}, \eqref{eq:prob_bd_by_exp}, and \eqref{eq:lip1} we conclude 
\begin{equation}
\label{eq:lip2}
\E[|\fn(t,\cp_{t,\ints})-\fn(t,\cp_{t,\alpha}) |]\le 2c\,\E\left[\int_0^t|\ints_r- \alpha_r|\diff r\right].
\end{equation}
Define the distance on the space of predictable nonnegative processes with bounded $L^1$ norm by
\begin{equation}
\label{eq:d_def}
d(\ints,\alpha):=\int_0^{\Time} e^{-2c t}\E[|\ints_t-\alpha_t|]\diff t.
\end{equation}
Then, by \eqref{eq:lip2} and integration by parts,
\begin{align*}
&d(\mathcal H(\ints),\mathcal H(\alpha))=\int_0^{\Time} e^{-4c t}\,\E[|\fn(t,\cp_{t,\ints})-\fn(t,\cp_{t,\alpha})|]\diff t\le  \int_0^{\Time}  2c \,e^{-4c t}\,\E\left[\int_0^t|\ints_r- \alpha_r|\diff r\right]\\
&=  \int_0^{\Time}  -\frac{1}{2}\frac{\diff}{\diff t} \,e^{-4c t}\,\E\left[\int_0^t|\ints_r- \alpha_r|\diff r\right]= \frac{1}{2}\int_0^{\Time}  e^{-4c t}\,\E\left[|\ints_t- \alpha_t|\diff t\right]=\frac{1}{2} d(\ints,\alpha),
\end{align*}
so $\mathcal H$ is $\frac{1}{2}$-Lipschitz with respect to $d$. Since the space of predictable nonnegative processes with bounded $L^1$ norm is complete with respect to $d$, it follows that $\mathcal H$ has a fixed point. 
\end{proof}

\section{Kolmogorov equations of the Poisson-F\"ollmer process}
\label{sec:kolmogorov}
The purpose of this section is to derive the forward and backward Kolmogorov equations of the Poisson-F\"ollmer process. These will have a number of important consequences for us. We begin with the forward equations. 

\begin{lemma}[Kolmogorov forward equations]
The translation probabilities $(\trans_{t|s})_{0\le s\le t\le \Time}$ of $(\cp_t)_{t\in [0,\Time]}$ satisfy, for all $\x_s\le\x_t\in\N^{\Dim}$,
\begin{equation}
\label{eq:Kolmogorov_forward_eq}
\begin{split}
\partial_t\trans_{t|s}(\x_t|\x_s)&=\sum_{\idx=1}^{\Dim}\left[\frac{\h(t,\x_t)}{h(t,\x_t-\basisi)}\trans_{t|s}(\x_t-\basisi|\x_s)-\frac{\h(t,\x_t+\basisi)}{h(t,\x_t)}\trans_{t|s}(\x_t|\x_s)\right]\\
&=\sum_{\idx=1}^{\Dim}\left[\ints^{\idx}(t,\x_t-\basisi)\trans_{t|s}(\x_t-\basisi|\x_s)-\ints^{\idx}(t,\x_t)\trans_{t|s}(\x_t|\x_s)\right],
\end{split}
\end{equation}
and the solution of \cref{eq:Kolmogorov_forward_eq} is
\begin{equation}
\label{eq:Kolmogorov_forward_sol}
\trans_{t|s}(\x_t|\x_s)=1_{\x_s\le \x_t}\frac{\h(t,\x_t)}{\h(s,\x_s)}\pos_{t-s}(\x_t-\x_s),\qquad \trans_{t=s|s}(\x_t|\x_s)=\delta_{\x_s}(\x_t),  \qquad \x_s\le\x_t\in\N^{\Dim}.
\end{equation}
\end{lemma}
\begin{proof}
Fix any $L^1(\pos_{\Time})$-integrable function $\fn:\N^{\Dim}\to \R$. The identity \eqref{eq:stoch_integ} can be written as 
\begin{equation}
\label{eq:stoch_integ_equiv}
\begin{split}
&\sum_{\x_t\in\N^{\Dim}}\fn(\x_t)\trans_{t|s}(\x_t|\x_s)-\fn(\x_s)=\\
&\sum_{\idx =1}^{\Dim}\int_s^t \left[\sum_{\x_r\in\N^{\Dim}}\fn(\x_r+e^{\idx})\frac{\h(r,\x_r+e^{\idx})}{\h(r,\x_r)}\trans_{r|s}(\x_r|\x_s)-\sum_{\x_r\in\N^{\Dim}}\fn(\x_r)\frac{\h(r,\x_r+e^{\idx})}{\h(r,\x_r)}\trans_{r|s}(\x_r|\x_s)\right]\diff r,
\end{split}
\end{equation}
so taking $G(x)=1_{x=y}$, for a fixed $y\in \N^{\Dim}$, gives 
\begin{equation}
\label{eq:stoch_integ_G_indicator}
\begin{split}
&\trans_{t|s}(y|\x_s)-1_{y=\x_s}\\
&=\sum_{\idx =1}^{\Dim}\int_s^t \left[\sum_{\x_r\in\N^{\Dim}}\frac{\h(r,y)}{\h(r,y-e^{\idx})}\trans_{r|s}(y-e^{\idx}|\x_s)-\sum_{\x_r\in\N^{\Dim}}\frac{\h(r,y+e^{\idx})}{\h(r,y)}\trans_{r|s}(y|\x_s)\right]\diff r.
\end{split}
\end{equation}
Differentiating \eqref{eq:stoch_integ_G_indicator} with respect to $t$ establishes \eqref{eq:Kolmogorov_forward_eq}, where we used \eqref{eq:stoch_int_appen}.

To verify \eqref{eq:Kolmogorov_forward_sol} fix $s,\x_s$, and for $0\le t\le \Time$ let
\[
g(t,\x):=1_{\x_s\le\x}\frac{\h(t,\x)}{\h(s,\x_s)}\pos_{t-s}(\x-\x_s).
\]
By \eqref{eq:Poisson_pde} and \eqref{eq:h_pde},
\begin{align*}
\partial_tg(t,\x)&=1_{\x_s\le\x}\frac{1}{\h(s,\x_s)}\sum_{\idx=1}^{\Dim}\left[-\h(t,\x)\der_{\idx}\pos_{t-s}(\x-\x_s-e^{\idx})-\der_{\idx}\h(t,\x)\pos_{t-s}(\x-\x_s)\right]\\
&=1_{\x_s\le\x}\frac{1}{\h(s,\x_s)}\sum_{\idx=1}^{\Dim}\left[-\h(t,\x)\pos_{t-s}(\x-\x_s)+\h(t,\x)\pos_{t-s}(\x-\x_s-e^{\idx})\right]\\
&+1_{\x_s\le\x}\frac{1}{\h(s,\x_s)}\sum_{\idx=1}^{\Dim}\left[-\h(t,\x+e^{\idx})\pos_{t-s}(\x-\x_s)+\h(t,\x)\pos_{t-s}(\x-\x_s)\right]\\
&=1_{\x_s\le\x}\frac{1}{\h(s,\x_s)}\sum_{\idx=1}^{\Dim}\left[\h(t,\x)\pos_{t-s}(\x-\x_s-e^{\idx})-\h(t,\x+e^{\idx})\pos_{t-s}(\x-\x_s)\right]\\
&=1_{\x_s\le\x}\frac{1}{\h(s,\x_s)}\sum_{\idx=1}^{\Dim}\left[\frac{\h(t,\x)}{\h(t,\x-e^{\idx})}\h(t,\x-e^{\idx})\pos_{t-s}(\x-\x_s-e^{\idx})-\frac{\h(t,\x+e^{\idx})}{\h(t,\x)}\h(t,\x)\pos_{t-s}(\x-\x_s)\right]\\
&=\sum_{\idx=1}^{\Dim}\left[\frac{\h(t,\x)}{\h(t,\x-e^{\idx})}g(t,\x-e^{\idx})-\frac{\h(t,\x+e^{\idx})}{\h(t,\x)}g(t,\x)\right],
\end{align*}
which is precisely \eqref{eq:Kolmogorov_forward_eq}. 
\end{proof}
A corollary of the Kolmogorov forward equation is that we can find an explicit form for the time marginals of the Poisson-F\"ollmer process.  In particular, note that the next result shows that at time $\Time$ the Poisson-F\"ollmer process follows the data distribution $\target$. 
\begin{corollary}[Time marginals]
\label{cor:time_marg}
The time marginals $(\flow_t)_{t\in [0,\Time]}$ of $(\cp_t)_{t\in [0,\Time]}$ satisfy the equation
\begin{equation}
\label{eq:time_marg_eq}
\begin{split}
\partial_t\flow_t(\x_t)=-\sum_{\idx=1}^{\Dim}\der_{\idx}[\ints^{\idx}(t,\x_t-\basisi)\flow_t(\x_t-\basisi)],
\end{split}
\end{equation}
and the solution to \cref{eq:time_marg_eq} is
\begin{equation}
\label{eq:time_marg}
\flow_t(\x_t)=\h(t,\x_t)\pos_t(\x_t),  \qquad \x_t\in\N^{\Dim}.
\end{equation}
\end{corollary}
\begin{proof}
Since $\cp_0=0$ we have $\flow_t(\x_t)=\trans_{t|0}(\x_t|0)$. Thus, \eqref{eq:time_marg_eq}-\eqref{eq:time_marg} follow by plugging $s=0$ into \eqref{eq:Kolmogorov_forward_eq}-\eqref{eq:Kolmogorov_forward_sol}, respectively, and using \eqref{eq:h_bd}. 
\end{proof}
Next we turn to the Kolmogorov backward equations.
\begin{lemma}[Kolmogorov backward equations]
\label{lem:Kolmogorov_backward_eq}
Fix $t\in [0,\Time]$ and let $\fn:\N^{\Dim}\to \R$ be an $L^1(\pos_{\Time})$-integrable function. For $s\in [0,t]$ and $\x_s\in\N^{\Dim}$ set
\[
\fn(s,\x_s):=\E[\fn(\cp_t)|\cp_s=\x_s], \qquad\qquad \fn(t,\x_t)=\fn(\x_t).
\]
Then,
\begin{equation}
\label{eq:Kolmogorov_backward_eq}
\partial_s\fn(s,\x_s)=-\sum_{\idx =1}^{\Dim}\frac{\h(s,\x_s+e^{\idx})}{\h(s,\x_s)}\der_{\idx}\fn(s,\x_s)=-\sum_{\idx=1}^{\Dim}\ints^{\idx}(s,\x_s)\der_{\idx}\fn(s,\x_s),
\end{equation}
and the solution of \cref{eq:Kolmogorov_backward_eq} is 
\begin{equation}
\label{eq:Kolmogorov_backward_sol}
\fn(s,\x_s)=\frac{\semi_{\Time-s}(\fn\h(t,\cdot))(\x_s)}{\h(s,\x_s)}.
\end{equation}
In particular,
\begin{equation}
\label{eq:Kolmogorov_backward_eq_trans}
\partial_s\trans_{t|s}(\x_t|\x_s)=-\sum_{\idx=1}^{\Dim}\frac{\h(s,\x_s+e^{\idx})}{\h(s,\x_s)}\der_{\idx}\trans_{t|s}(\x_t|\x_s)=-\sum_{\idx=1}^{\Dim}\ints^{\idx}(s,\x_s)\der_{\idx}\trans_{t|s}(\x_t|\x_s).
\end{equation}
\end{lemma}
\begin{proof}
Let $\tilde{\fn}$ be a solution of \eqref{eq:Kolmogorov_backward_eq}.  Then, by \eqref{eq:Ito}, 
\begin{align*}
&\tilde{\fn}(s,\cp_s)=\fn(0,\cp_0)+\int_0^s\partial_r\tilde{\fn}(r,\cp_r)\diff r+\int_0^s\sum_{\idx=1}^{\Dim}\der_{\idx}\tilde{\fn}(r,\cp_r)\cp^{\idx}(\diff r)\\
&=\fn(0,\cp_0)+\int_0^s\left[\partial_r\tilde{\fn}(r,\cp_r)+\sum_{\idx=1}^{\Dim}\der_{\idx}\tilde{\fn}(r,\cp_r)\ints^{\idx}(r,\cp_r)\right]\diff r+\int_0^s\sum_{\idx=1}^{\Dim}\der_{\idx}\tilde{\fn}(r,\cp_r)[\cp^{\idx}(\diff r)-\ints^{\idx}(r,\cp_r)\diff r]\\
&=\fn(0,\cp_0)+0+\int_0^s\sum_{\idx=1}^{\Dim}\der_{\idx}\tilde{\fn}(r,\cp_r)[\cp^{\idx}(\diff r)-\ints^{\idx}(r,\cp_r)\diff r],
\end{align*}
which shows that $s\mapsto \tilde{\fn}(s,\cp_s)$ is a martingale with respect to the filtration generated by $\cp$, satisfying $\tilde{\fn}(t,\cp_t)=\fn(\cp_t)$. In particular,
\[
\tilde{\fn}(s,\cp_s)=\E[\tilde{\fn}(t,\cp_t)|\cp_s]=\E[\fn(\cp_t)|\cp_s]=\fn(s,\cp_s),
\]
which proves that $\fn(s,\x_s)$ satisfies \eqref{eq:Kolmogorov_backward_eq}.

To verify \eqref{eq:Kolmogorov_backward_sol} let $g(s,\x):=\frac{\semi_{\Time-s}(\fn\h(t,\cdot))(\x)}{\h(s,\x)}$ and compute, by \eqref{eq:Poisson_semi_pde} and \eqref{eq:h_pde},
\begin{align*}
\partial_sg(s,\x)&=\frac{\partial_s\semi_{\Time-s}(\fn\h(t,\cdot))(\x)}{\h(s,\x)}-\frac{\semi_{\Time-s}(\fn\h(t,\cdot))(\x)}{\h(s,\x)^2}\partial_s\h(s,\x)\\
&=\sum_{\idx=1}^{\Dim}\left[\frac{-\der_{\idx}\semi_{\Time-s}(\fn\h(t,\cdot))(\x)}{\h(s,\x)}+\frac{\semi_{\Time-s}(\fn\h(t,\cdot))(\x)}{\h(s,\x)^2}\der_{\idx}\h(s,\x)\right]\\
&=\sum_{\idx=1}^{\Dim}\frac{\semi_{\Time-s}(\fn\h(t,\cdot))(\x)-\semi_{\Time-s}(\fn\h(t,\cdot))(\x+\basisi)}{\h(s,\x)}+\sum_{\idx=1}^{\Dim}\frac{\semi_{\Time-s}(\fn\h(t,\cdot))(\x)[\h(s,\x+\basisi)-\h(s,x)]}{\h(s,\x)^2}\\
&=\sum_{\idx=1}^{\Dim}\left\{\frac{\semi_{\Time-s}(\fn\h(t,\cdot))(\x)\h(s,\x+\basisi)}{\h(s,\x)^2}\frac{-\semi_{\Time-s}(\fn\h(t,\cdot))(\x+\basisi)}{\h(s,\x)}\right\}. 
\end{align*}
On the other hand, for  each $\idx=1,\ldots,\Dim$, 
\begin{align*}
\frac{\h(s,\x+e^{\idx})}{\h(s,\x)}\der_{\idx}g(s,\x)&=\frac{\h(s,\x+e^{\idx})}{\h(s,\x)}\left\{\frac{\semi_{\Time-s}(\fn\h(t,\cdot))(\x+e^{\idx})}{\h(s,\x+e^{\idx})}-\frac{\semi_{\Time-s}(\fn\h(t,\cdot))(\x)}{\h(s,\x)}\right\}\\
&=\frac{\semi_{\Time-s}(\fn\h(t,\cdot))(\x+e^{\idx})}{\h(s,\x)}-\frac{\semi_{\Time-s}(\fn\h(t,\cdot))(\x)\h(s,\x+e^{\idx})}{\h(s,\x)^2},
\end{align*}
which completes the proof of \eqref{eq:Kolmogorov_backward_sol}. 

Equation \eqref{eq:Kolmogorov_backward_eq_trans} follows from \eqref{eq:Kolmogorov_backward_eq} by noting that 
\[
\trans_{t|s}(\x_t|\x_s)=\E[1_{\cp_t=\x_t}|\cp_s=\x_s].
\]
\end{proof}

\section{Proofs}
\label{sec:proofs}
In this section we prove Proposition \ref{prop:key_properties} and Proposition \ref{prop:loglike}.

\subsection{Proof of Proposition \ref{prop:key_properties}}
\label{subsec:proof_key_prop}
Property (a) follows from Corollary \ref{cor:time_marg} with $t=\Time$. The next lemma establishes property (b). 
\begin{lemma}[Bridges of Poisson-F\"ollmer process]
\label{lem:bridge}
Let $(\cp_t)_{t\in [0,\Time]}$ be the Poisson-F\"ollmer process. Then,
$\cp_t|\cp_{\Time}\sim \B_{\cp_{\Time},\frac{t}{\Time}}$.
\end{lemma}
\begin{proof}
Fix $x,y\in\N^{\Dim}$ with $y\ge x$. We have
\begin{align*}
&\trans_{\Time|t}(y|\x)\overset{\eqref{eq:Kolmogorov_forward_sol}}{=}\frac{\h(\Time,y)}{\h(t,\x)}\pos_{\Time-t}(y-\x)\overset{\eqref{eq:h_bd}}{=}\frac{\reltarget(y)}{\h(t,\x)}\pos_{\Time-t}(y-\x),\\
&\flow_t(\x)\overset{\eqref{eq:time_marg}}{=}\h(t,\x)\pos_t(\x),\qquad \flow_{\Time}(y)\overset{\eqref{eq:time_marg}-\eqref{eq:h_bd}}{=}\reltarget(y)\pos_{\Time}(y),
\end{align*} 
so by Bayes' rule
\begin{align*}
\trans_{t|\Time}(\x|y)&=\trans_{\Time|t}(y|\x)\frac{\flow_t(\x)}{\pos_{\Time}(y)}=\frac{\reltarget(y)}{\h(t,\x)}\pos_{\Time-t}(y-\x)\frac{h(t,\x)\pos_t(\x)}{\reltarget(y)\pos_{\Time}(y)}=\frac{\pos_{\Time-t}(y-\x)\pos_t(\x)}{\pos_{\Time}(y)}\\
&=\prod_{\idx=1}^{\Dim} {y^{\idx}\choose \x^{\idx}}\left(\frac{t}{\Time}\right)^{\x^{\idx}}\left(\frac{\Time -t}{\Time}\right)^{y^{\idx}-\x^{\idx}}=\B_{y,\frac{t}{\Time}}(\x).
\end{align*}
\end{proof}
The next lemma establishes property (c). 
\begin{lemma}[Relation between denoiser and rate]
\label{lem:denoiser_rate}
Let $(\cp_t)_{t\in [0,\Time]}$ be the Poisson-F\"ollmer process. Then,
\begin{equation}
\label{eq:denoiser_rate}
\ints(t,\cp_t)=\frac{\E[\cp_{\Time}|\cp_t]-\cp_t}{\Time -t}. 
\end{equation}
\end{lemma}
\begin{proof}
Fix $\idx\in [\Dim]$. Applying \eqref{eq:Kolmogorov_backward_sol} with $\fn(s,\x_s):=\x_s^{\idx}$ we have 
\begin{align*}
\E[\cp_{\Time}^{\idx}|\cp_t]=\frac{1}{\h(t,\cp_T)}\sum_{y\in \N^{\Dim}} \reltarget(\cp_t+y)(\cp_t^{\idx}+y^{\idx})\pos_{\Time-t}(y)=\cp_t^{\idx}+\frac{1}{\h(t,\cp_T)}\sum_{y\in \N^{\Dim}} \reltarget(\cp_t+y)y^{\idx}\pos_{\Time-t}(y),
\end{align*}
where the last equality used the definition of $\h$. Define $\reltarget_{\x^{-\idx},y^{-\idx}}:\N\to\R$ by
\[
\reltarget_{\x^{-\idx},y^{-\idx}}(k):=\reltarget(\x^{-\idx}+y^{-\idx},k),
\]
and compute
\begin{align*}
&\sum_{y\in \N^{\Dim}} \reltarget(\cp_t+y)y^{\idx}\pos_{\Time-t}(y)=\sum_{y^{-\idx}\in\N^{\Dim-1}}\left(\prod_{j\neq \idx}\pos_{\Time-t}(y^j)\right)\sum_{y^{\idx}\in\N}\reltarget_{\x^{-\idx},y^{-\idx}}(\cp_t^{\idx}+y^{\idx})y^{\idx}e^{-(\Time -t)}\frac{(\Time -t)^{y^{\idx}}}{y^{\idx}!}\\
&=(\Time-t)\sum_{y^{-\idx}\in\N^{\Dim-1}}\left(\prod_{j\neq \idx}\pos_{\Time-t}(y^j)\right)\sum_{y^{\idx}\in\N}\reltarget_{\x^{-\idx},y^{-\idx}}((\cp_t^{\idx}+1)+(y^{\idx}-1))e^{-(\Time -t)}\frac{(\Time -t)^{y^{\idx}-1}}{(y^{\idx}-1)!}\\
&=(\Time-t)\sum_{y^{-\idx}\in\N^{\Dim-1}}\left(\prod_{j\neq \idx}\pos_{\Time-t}(y^j)\right)\semi_{\Time-t}\reltarget_{\x^{-\idx},y^{-\idx}}(\cp_t^{\idx}+1)\\
&=(\Time-t)\sum_{y^{-\idx}\in\N^{\Dim-1}}\left(\prod_{j\neq \idx}\pos_{\Time-t}(y^j)\right)\sum_{y^{\idx}\in\N}\reltarget((\x^{-\idx}+y^{-\idx},\cp_t^{\idx}+y^{\idx}+1))\pos_{\Time-t}(y^{\idx})\\
&=(\Time -t)\semi_{\Time-t}\reltarget(\cp_t+\basisi)=(\Time -t)\h(t,\cp_t+\basisi).
\end{align*}
Thus, 
\[
\E[\cp_{\Time}^{\idx}|\cp_t]=\cp_t^{\idx}+\frac{(\Time -t)\h(t,\cp_t+\basisi)}{\h(t,\cp_t)},
\]
and the result follows from \eqref{eq:stoch_int_appen}. 
\end{proof}

\subsection{Proof of Proposition \ref{prop:loglike}}
\label{subsec:proof_log_likelihood}
To simplify the notation we will prove the  proposition for $\Time=1$, but the proof for the general case is similar. We recall that for any probability measures $\target,\nu$ on $\N^{\Dim}$ we have
\begin{equation}
\label{eq:KL_Def_appen}
\ent(\nu|\target):=\sum_{\x\in \N^{\Dim}}\log\left(\frac{\nu(x)}{\target(x)}\right)\nu(x).
\end{equation}
The proof of Proposition \ref{prop:loglike} will follow from the following lemma.
\begin{lemma}
\label{lem:time_der_rel_ent}
Let
\begin{equation}
\label{eq:cvx_loglike_append}
\cvx(a):=\sum_{\idx=1}^{\Dim}a^{\idx}\log a^{\idx},\qquad \nabla \cvx(a):=(1+\log a^1,\ldots, 1+\log a^{\Dim}),\qquad a\in \N^{\Dim},
\end{equation}
and
\begin{equation}
\label{eq:beg_div_def_append}
\begin{split}
\beg(a,b)&:=\cvx(a)-\cvx(b)-\nabla \cvx(b)\cdot(a-b)\\
&=\sum_{\idx=1}^{\Dim}\left[a^{\idx}\log a^{\idx} - a^{\idx} \log b^{\idx}-a^{\idx}+b^{\idx}\right],\qquad a,b\in \N^{\Dim}.
\end{split}
\end{equation}
Then, for $\x\in \N^{\Dim}$,
\begin{equation}
\label{eq:dert_log_kl_ell_lem}
\begin{split}
&\partial_t\ent(\trans_{t|1}(\cdot|\x)|\flow_t)=\sum_{y\in\N^{\Dim}}\beg\left(\frac{\x-y}{1-t},\ints(t,y)\right)\B_{\x,t}(y).
\end{split}
\end{equation}
\end{lemma}
Let us prove Proposition \ref{prop:loglike} assuming 
 Lemma \ref{lem:time_der_rel_ent}.
\begin{proof}[Proof of Proposition \ref{prop:loglike}]
Since
\begin{equation}
\label{eq:bd_flowtrans}
\flow_0=\delta_0,\qquad \flow_1=\target, \qquad \trans_{0|1}(\cdot|\x)=\B_{\x,0}=\delta_0,\qquad \trans_{1|1}(\cdot|\x)=\B_{\x,1}=\delta_{\x},
\end{equation}
we get 
\begin{align*}
&-\log\target(x)=\ent(\delta_{\x}|\target)\overset{\eqref{eq:bd_flowtrans}}{=}\ent(\trans_{1|1}(\cdot|\x)|\flow_1)-\ent(\trans_{0|1}(\cdot|\x)|\flow_0)=\int_0^1\partial_t\ent(\trans_{t|1}(\cdot|\x)|\flow_t)\diff t\\
&\overset{\eqref{eq:dert_log_kl_ell_lem}}{=}\int_0^1\sum_{y\in\N^{\Dim}}\beg\left(\frac{\x-y}{1-t},\ints(t,y)\right)\B_{\x,t}(y)\diff t.
\end{align*}
\end{proof}
The remainder of the section is dedicated to the proof of Lemma \ref{lem:time_der_rel_ent}. We begin with some useful identities for the Binomial distribution, $\Bin_{\x,t}:=\B_{\x,t}$.
\begin{claim}
Fix $x\in \N^{\Dim}$ and $t\in [0,1]$. Then, for any $y\le x$, 
\begin{equation}
\label{eq:dert_log_B}
\partial_t\log \Bin_{\x,t}(y)=\sum_{\idx=1}^{\Dim}\left\{\frac{y^{\idx}}{t} -\frac{\x^{\idx}-y^{\idx}}{1-t}\right\},
\end{equation}
\begin{equation}
\label{eq:iden_B}
\frac{y^{\idx}}{t}\Bin_{\x,t}(y)=\x^{\idx}\Bin_{\x-\basisi,t}(y-\basisi),\qquad \frac{\x^{\idx}-y^{\idx}}{1-t}\Bin_{\x,t}(y)=\x^{\idx}\Bin_{\x-\basisi,t}(y), 
\end{equation}
and
\begin{equation}
\label{eq:dert_B}
\partial_t \Bin_{\x,t}(y)=\sum_{\idx=1}^{\Dim}\x^{\idx}\left\{\Bin_{\x-\basisi,t}(y-\basisi)-\Bin_{\x-\basisi,t}(y)\right\}.
\end{equation}
\end{claim}

\begin{proof}
Equation \eqref{eq:dert_log_B} follows from
\begin{align*}
\partial_t\log \Bin_{\x,t}(y)&=\sum_{\idx=1}^{\Dim}\partial_t\log\left\{{\x^{\idx} \choose y^{\idx}}+y^{\idx}\log t+(\x^{\idx}-y^{\idx})\log(1-t)\right\}=\sum_{\idx=1}^{\Dim}\left\{\frac{y^{\idx}}{t} -\frac{\x^{\idx}-y^{\idx}}{1-t}\right\}.
\end{align*}
The first equation in \eqref{eq:iden_B} follows from 
\begin{align*}
&\frac{y^{\idx}}{t}\Bin_{\x,t}(y)=\left(\prod_{j\neq \idx}\Bin_{\x^j,t}(y^j)\right)\frac{y^{\idx}}{t}\left\{\frac{\x^{\idx}!}{y^{\idx}!(\x^{\idx}-y^{\idx})!}t^{y^{\idx}}(1-t)^{\x^{\idx}-y^{\idx}}\right\}\\
&=\left(\prod_{j\neq \idx}\Bin_{\x^j,t}(y^j)\right)\x^{\idx}\left\{\frac{(\x^{\idx}-1)!}{(y^{\idx}-1)!(\x^{\idx}-1-(y^{\idx}-1))!}t^{y^{\idx}-1}(1-t)^{\x^{\idx}-1-(y^{\idx}-1)}\right\}\\
&=\left(\prod_{j\neq \idx}\Bin_{\x^j,t}(y^j)\right)\x^{\idx}\Bin_{\x^{\idx}-1,t}(y^{\idx}-1)=\x^{\idx}\Bin_{\x-\basisi,t}(y-\basisi),
\end{align*}
and the second equation in \eqref{eq:iden_B} follows as
\begin{align*}
&\frac{\x^{\idx}-y^{\idx}}{1-t}\Bin_{\x,t}(y)=\left(\prod_{j\neq \idx}\Bin_{\x^j,t}(y^j)\right)\left\{\frac{\x^{\idx}-y^{\idx}}{1-t}\frac{\x^{\idx}!}{y^{\idx}!(\x^{\idx}-y^{\idx})!}t^{y^{\idx}}(1-t)^{\x^{\idx}-y^{\idx}}\right\}\\
&=\left(\prod_{j\neq \idx}\Bin_{\x^j,t}(y^j)\right)\x^{\idx}\left\{\frac{(\x^{\idx}-1)!}{y^{\idx}!(\x^{\idx}-1-y^{\idx})!}t^{y^{\idx}}(1-t)^{\x^{\idx}-1-y^{\idx}}\right\}=\left(\prod_{j\neq \idx}\Bin_{\x^j,t}(y^j)\right)\x^{\idx}\Bin_{\x^{\idx}-1,t}(y^{\idx})\\
&=\x^{\idx}\Bin_{\x-\basisi,t}(y).
\end{align*}
Equation \eqref{eq:dert_B} follows from the combination of \eqref{eq:dert_log_B} and \eqref{eq:iden_B},
\begin{align*}
\partial_t \Bin_{\x,t}(y)&=\sum_{\idx=1}^{\Dim}\x^{\idx}\left\{\Bin_{\x-\basisi,t}(y-\basisi)-\Bin_{\x-\basisi,t}(y)\right\}.
\end{align*}
\end{proof}
We now turn to the proof of Lemma \ref{lem:time_der_rel_ent}. By definition, 
\begin{align*}
&\ent(\trans_{t|\Time}(\cdot|\x)|\flow_t)=\sum_{y\in\N^{\Dim}}\left[\log\Bin_{\x,t}(y)-\log\flow_t(y)\right]\Bin_{\x,t}(y),
\end{align*}
so 
\begin{equation}
\label{eq:dert_log_kl}
\partial_t\ent(\trans_{t|\Time}(\cdot|\x)|\flow_t)=\sum_{y\in\N^{\Dim}}\log\left(\frac{\Bin_{\x,t}(y)}{\flow_t(y)}\right)\partial_t\Bin_{\x,t}(y)+\sum_{y\in\N^{\Dim}}\left[\partial_t\log\Bin_{\x,t}(y)-\partial_t\log\flow_t(y)\right]\Bin_{\x,t}(y).
\end{equation}
We will analyze the two terms in \eqref{eq:dert_log_kl} in the following two claims. 
\begin{claim}
\label{cl:term1}
\begin{equation}
\label{eq:term_A}
\sum_{y\in\N^{\Dim}}\log\left(\frac{\Bin_{\x,t}(y)}{\flow_t(y)}\right)\partial_t\Bin_{\x,t}(y)=\sum_{\idx=1}^{\Dim}\sum_{y\in\N^{\Dim}}\left\{\log\left(\frac{\x^{\idx}-y^{\idx}}{1-t}\right)\frac{\x^{\idx}-y^{\idx}}{1-t}-\frac{\x^{\idx}-y^{\idx}}{1-t}\log\ints^{\idx}(t,y)\right\}\Bin_{\x,t}(y).
\end{equation}
\end{claim}

\begin{proof}
Re-indexing we get,
\begin{align*}
&\sum_{y\in\N^{\Dim}}\log\left(\frac{\Bin_{\x,t}(y)}{\flow_t(y)}\right)\partial_t\Bin_{\x,t}(y)\overset{\eqref{eq:dert_B}}{=}\sum_{\idx=1}^{\Dim}\sum_{y\in\N^{\Dim}}\log\left(\frac{\Bin_{\x,t}(y)}{\flow_t(y)}\right)\x^{\idx}\Bin_{\x-\basisi,t}(y-\basisi)-\sum_{\idx=1}^{\Dim}\sum_{y\in\N^{\Dim}}\log\left(\frac{\Bin_{\x,t}(y)}{\flow_t(y)}\right)\x^{\idx}\Bin_{\x-\basisi,t}(y)\\
&=\sum_{\idx=1}^{\Dim}\sum_{y\in\N^{\Dim}}\log\left(\frac{\Bin_{\x,t}(y+\basisi)}{\Bin_{\x,t}(y)}\frac{\flow_t(y)}{\flow_t(y+\basisi)}
\right)\x^{\idx}\Bin_{\x-\basisi,t}(y)\overset{\eqref{eq:iden_B}}{=}\sum_{\idx=1}^{\Dim}\sum_{y\in\N^{\Dim}}\log\left(\frac{\Bin_{\x,t}(y+\basisi)}{\Bin_{\x,t}(y)}\frac{\flow_t(y)}{\flow_t(y+\basisi)}
\right)\frac{\x^{\idx}-y^{\idx}}{1-t}\Bin_{\x,t}(y).
\end{align*}
Using
\begin{equation}
\label{eq:ratio_den}
\frac{\Bin_{\x,t}(y+\basisi)}{\Bin_{\x,t}(y)}=\frac{\x^{\idx}-y^{\idx}}{y^{\idx}+1}\frac{t}{1-t},\qquad \qquad
\frac{\flow_t(y)}{\flow_t(y+\basisi)}\overset{\eqref{eq:time_marg}}{=}\frac{\h(t,y)\pos_t(y)}{\h(t,y+\basisi)\pos_t(y+\basisi)}\overset{\eqref{eq:stoch_int_appen}}{=}\frac{1}{\ints^{\idx}(t,y)}\frac{y^{\idx}+1}{t},
\end{equation}
we get 
\[
\frac{\Bin_{\x,t}(y+\basisi)}{\Bin_{\x,t}(y)}\frac{\flow_t(y)}{\flow_t(y+\basisi)}\overset{\eqref{eq:ratio_den}}{=}\frac{1}{\ints^{\idx}(t,y)}\frac{\x^{\idx}-y^{\idx}}{1-t},
\]
and thus 
\begin{align*}
&\sum_{y\in\N^{\Dim}}\log\left(\frac{\Bin_{\x,t}(y)}{\flow_t(y)}\right)\partial_t\Bin_{\x,t}(y)=\sum_{\idx=1}^{\Dim}\sum_{y\in\N^{\Dim}}\left\{\log\left(\frac{\x^{\idx}-y^{\idx}}{1-t}\right)\frac{\x^{\idx}-y^{\idx}}{1-t}-\frac{\x^{\idx}-y^{\idx}}{1-t}\log\ints^{\idx}(t,y)\right\}\Bin_{\x,t}(y).
\end{align*}
\end{proof}
We now turn to the second term in \eqref{eq:dert_log_kl}. 
\begin{claim}
\label{cl:term2}
\begin{equation}
\label{eq:term_B}
\sum_{y\in\N^{\Dim}}\left[\partial_t\log\Bin_{\x,t}(y)-\partial_t\log\flow_t(y)\right]\Bin_{\x,t}(y)=\sum_{\idx=1}^{\Dim}\sum_{y\in\N^{\Dim}}\left[\frac{\x^{\idx}-y^{\idx}}{1-t}+\ints^{\idx}(t,y)\right]\Bin_{\x,t}(y).
\end{equation}
\end{claim}
\begin{proof}
Since
\begin{equation}
\label{eq:der_t:log_pt}
\partial_t\log\flow_t(y)\overset{\eqref{eq:time_marg}}{=}\partial_t\log\h(t,y)+\partial_t\log\pos_t(y)\overset{\eqref{eq:h_pde}}{=}\sum_{\idx=1}^{\Dim}\left[1-\frac{\h(t,t+\basisi)}{\h(t,y)}\right]+\sum_{\idx=1}^{\Dim}\left\{-1+\frac{y^{\idx}}{t}\right\}\overset{\eqref{eq:stoch_int_appen}}{=}\sum_{\idx=1}^{\Dim}\left[-\ints^{\idx}(t,y)+\frac{y^{\idx}}{t}\right],
\end{equation}
we get
\begin{align*}
&\sum_{y\in\N^{\Dim}}\left[\partial_t\log\Bin_{\x,t}(y)-\partial_t\log\flow_t(y)\right]\Bin_{\x,t}(y)\overset{\eqref{eq:dert_log_B},\eqref{eq:der_t:log_pt}}{=}\sum_{y\in\N^{\Dim}}\sum_{\idx=1}^{\Dim}\left[\left\{\frac{y^{\idx}}{t} -\frac{\x^{\idx}-y^{\idx}}{1-t}\right\}-\left\{-\ints^{\idx}(t,y)+\frac{y^{\idx}}{t}\right\}\right]\Bin_{\x,t}(y)\\
&=\sum_{y\in\N^{\Dim}}\sum_{\idx=1}^{\Dim}\left[\ints^{\idx}(t,y)-\frac{\x^{\idx}-y^{\idx}}{1-t}\right]\Bin_{\x,t}(y).
\end{align*}
\end{proof}
With Claim \ref{cl:term1} and Claim \ref{cl:term2} in hand Equation \eqref{eq:dert_log_kl} reads
\begin{equation}
\label{eq:dert_log_kl_ell}
\begin{split}
&\partial_t\ent(\trans_{t|1}(\cdot|\x)|\flow_t)=\sum_{\idx=1}^{\Dim}\sum_{y\in\N^{\Dim}}\left\{
\log\left(\frac{1}{\ints^{\idx}(t,y)}\frac{\x^{\idx}-y^{\idx}}{1-t}\right)\frac{\x^{\idx}-y^{\idx}}{1-t}-\frac{\x^{\idx}-y^{\idx}}{1-t}+\ints^{\idx}(t,y)\right\}\Bin_{\x,t}(y)\\
&=\sum_{\idx=1}^{\Dim}\sum_{y\in\N^{\Dim}}\left\{
\log\left(\frac{\x^{\idx}-y^{\idx}}{1-t}\right)\frac{\x^{\idx}-y^{\idx}}{1-t}-\frac{\x^{\idx}-y^{\idx}}{1-t}\log\ints^{\idx}(t,y)-\frac{\x^{\idx}-y^{\idx}}{1-t}+\ints^{\idx}(t,y)\right\}\Bin_{\x,t}(y)\\
&=\sum_{y\in\N^{\Dim}}\beg\left(\frac{\x-y}{1-t},\ints(t,y)\right)\Bin_{\x,t}(y),
\end{split}
\end{equation}
where  $\cvx$ is as in \eqref{eq:cvx_loglike_append} and $\beg$ as in \eqref{eq:beg_div_def_append}.

\section{Schr\"odinger bridges}
\label{sec:schrondinger}
In this section we explain how the Poisson-F\"ollmer process is a special case of a \Def{Schr\"odinger bridge}. Let $\paths$ be the space of functions from $[0,\Time]$ to $\N^{\Dim}$, and let $\Rp$ be the law of the standard $\Dim$-dimensional Poisson process on $\N^{\Dim}$. Given a probably measure $\Qp\ll\Rp$ on $\paths$ we denote the Kullback-Leibler (KL) divergence between $\Qp$ and $\Rp$ as
\begin{equation}
\label{eq:KL_def}
\ent(\Qp|\Rp):=\int_{\paths} \log \left(\frac{\diff \Qp}{\diff \Rp}\right)\diff \Qp.
\end{equation}
For $t\in [0,\Time]$ we let $\Rp_t,\Qp_t$ be the marginals of the coordinate process at time $t$ with respect to $\Rp,\Qp$, respectively.
\begin{definition}[Schr\"odinger problem for Dirac source]
\label{def:schrdoinger_problem}
Let $\Rp$  be the Poisson measure over $\paths$ and let $\target$ be a distribution over $\N^{\Dim}$. Solve
\begin{equation}
\label{eq:schrdoinger_argmin}
\min_{\Qp}\, \ent(\Qp|\Rp)\quad\textnormal{such that} \quad \Qp_0\sim \delta_0, \quad \Qp_{\Time}\sim\target.
\end{equation}
\end{definition}
The problem introduces in Definition \ref{def:schrdoinger_problem} is a special case of the \Def{Schr\"odinger problem}. In the general problem the source measure need not be a point mass, but for our purpose of designing a simple recipe for denoising and sampling, the problem \eqref{eq:schrdoinger_argmin} is better suited. Our goal in this section is to show that the  Poisson-F\"ollmer process $(\cp_t)_{t\in [0,\Time]}$ is an optimal solution to the Schr\"odinger problem for Dirac source. The next result gives a sufficient condition for a measure $\Qp$ to be an optimal solution to \eqref{eq:schrdoinger_argmin}. 
\begin{lemma}
If $\Qp$ satisfies
\begin{equation}
\label{eq:schrdoinger_argmin_opt}
\ent(\Qp|\Rp)=\ent(\target|\pos_{\Time}),
\end{equation}
then  $\Qp$ is an optimal solution to \eqref{eq:schrdoinger_argmin}. 
\end{lemma}
\begin{proof}
Since the $\Time$-marginal under $\Rp$ is $\pos_{\Time}$, the chain rule  for entropy gives
\begin{equation}
\label{eq:ent_chain_rule}
\ent(\Qp|\Rp)=\ent(\target|\pos_{\Time})+\sum_{\x\in\N^{\Dim}}\ent(\Qp^{\x}|\Rp^{\x})\diff\target(\x),
\end{equation}
where  $\Rp^{\x}$ (res. $\Qp^{\x}$) is the bridge of $\Rp$ (res. $\Qp$) starting at $0$ and terminating at $\x$. The term $\ent(\target|\pos_{\Time})$ is the same for any $\Qp$ satisfying $\Qp_0\sim \delta_0,  \Qp_{\Time}\sim\target$, which completes the argument.  
\end{proof}
Our next result shows that if $\Qp$ is the law of the Poisson-F\"ollmer process then it  satisfies \eqref{eq:schrdoinger_argmin_opt}, and hence it is an optimal solution to the Schr\"odinger problem. 
\begin{lemma}
The  law $\Qp$ of the Poisson-F\"ollmer process satisfies
\begin{equation}
\label{eq:ent_min}
\ent(\Qp|\Rp)=\sum_{\idx=1}^{\Dim}\int_0^{\Time}\E\left[\ints^{\idx}(t,\cp_t)\log\ints^{\idx}(t,\cp_t)-\ints^{\idx}(t,\cp_t)+1\right]\diff t=\ent(\target|\pos_{\Time}).
\end{equation}
\end{lemma}
\begin{proof}
The first equality in \eqref{eq:ent_min} is a consequence of Girsanov's theorem for jump processes, e.g. Equation (14) in \cite{MR3646434}, and Equation \eqref{eq:stoch_to_determ}. To establish the second equality in \eqref{eq:ent_min} we will show that
\begin{equation}
\label{eq:dertKL_marg_ref}
\partial_t\ent(\flow_t|\pos_t)=\sum_{\idx=1}^{\Dim}\E[\ints^{\idx}(t,\cp_t)\log \ints^{\idx}(t,\cp_t)-\ints^{\idx}(t,\cp_t)+1].
\end{equation}
Once \eqref{eq:dertKL_marg_ref} is established we can deduce the result since $\pos_0=\flow_0=\delta_0$ and $\flow_{\Time}=\target$. To prove \eqref{eq:dertKL_marg_ref} we compute 
\begin{align*}
&\partial_t\ent(\flow_t|\pos_t)\overset{\eqref{eq:time_marg}}{=}\partial_t\sum_{y\in \N^{\Dim}}[\log \h(t,y) ] \flow_t(y)\\
&\overset{\eqref{eq:h_log_pde},\eqref{eq:time_marg_eq}}{=}\sum_{y\in \N^{\Dim}}\sum_{\idx=1}^{\Dim}\left[1-\ints^{\idx}(t,y)\right] \flow_t(y)-\sum_{\idx=1}^{\Dim}\sum_{y\in \N^{\Dim}}\log \h(t,y) \der_{\idx}[\ints^{\idx}(t,y-\basisi)\flow_t(y-\basisi)]\\
&\overset{\eqref{eq:sum_by_parts}}{=}\sum_{y\in \N^{\Dim}}\sum_{\idx=1}^{\Dim}\left\{1-\ints^{\idx}(t,y) +\der_{\idx}\log \h(t,y) \ints^{\idx}(t,y)\right\}\flow_t(y)\overset{\eqref{eq:stoch_int_appen}}{=}\sum_{y\in \N^{\Dim}}\sum_{\idx=1}^{\Dim}\left\{1-\ints^{\idx}(t,y) +[\log \ints^{\idx}(t,y)]\ints^{\idx}(t,y)\right\}\flow_t(y).
\end{align*}
\end{proof}

\section{Time reversal}
\label{sec:time_rev}
The purpose of this section is to explain how our constructions is related to time-reversal of discrete diffusion models. Define the \Def{forward process} $\vec{Z}_t:=\cp_{\Time-t}$ to be a \Def{Poisson bridge} which starts at $\vec{Z}_0\sim\target$ and terminates at $\vec{Z}_{\Time}=0$. Then the Poisson-F\"ollmer process  $\cp_t=:\cev{Z_t}$ is the \Def{reverse process} obtained by the time-reversal  of $\vec{Z}_t$. Hence, from the perspective of the time-reversal approach in discrete diffusions \cite{lou2024discrete} our goal should be learn the discrete score $\left[\frac{q_t(y)}{q_t(\x)}\right]_{y\neq \x}$, where $q_t$ is the distribution of $\vec{Z}_t$. Since the forward process $(\vec{Z}_t)$ decreases in increments of 1 in each coordinate, its rate matrices $(\vec{Q}_t)$ take the form
\begin{equation}
\label{eq:forward_diff_rate_mat}
\vec{Q}_t(\x,y)=1_{y\in \{\x-\basis_1,\ldots,\x-\basis_{\Dim}\}}\, \frac{1}{\Dim},\qquad\qquad y\neq x,
\end{equation}
so the rate matrices $(\cev{Q}_t)$ of the reverse process $(\cev{Z_t})$ are
\begin{equation}
\label{eq:reverse_diff_rate_mat}
\vec{Q}_t(\x,x+\basisi)= \frac{1}{\Dim}\frac{q_t(\x+\basisi)}{q_t(x)},\qquad\qquad \vec{Q}_t(\x,y)=0\textnormal{ for }y\notin \{\x,\x+\basis_1,\ldots,\x+\basis_{\Dim}\}.
\end{equation}
Since $q_t=\flow_{\Time-t}$, Corollary \ref{cor:time_marg} shows that
\begin{equation}
\label{eq:ints_time_reversal}
\frac{q_t(\x+\basisi)}{q_t(x)}=\frac{\h(\Time -t,\x+\basisi)\pos_{\Time -t}(\x+\basisi)}{\h(\Time -t,\x)\pos_{\Time -t}(\x)}=\frac{\Time-t}{\x^i+1}\ints^i(\Time-t,\x).
\end{equation}
Thus, learning the intensity $\ints$ of the Poisson-F\"ollmer process is equivalent to learning the time-reversal of the forward process $(\vec{Z}_t)$.

\section{Gaussian analogues: The F\"ollmer process}
\label{sec:gaussian}
The continuous analogue of the Poisson-F\"ollmer process is the \Def{F\"ollmer process}. To describe this process we write the data distribution as 
\begin{equation}
\label{eq:data_gauss}
\target=\reltarget\gauss_{0,\Time},
\end{equation}
where $\gauss_{x,t}$ is the Gaussian with mean $x\in\R^{\Dim}$ and covariance $t\id_{\Dim}$ on $\R^{\Dim}$. The \Def{heat semigroup} is acting on $L^1(\gauss_{\Time})$-integrable functions $G:\R^{\Dim}\to\R$ by 
\begin{equation}
\label{eq:heat}
\heat_t\fn(\x):=\int_{\R^{\Dim}}\fn(\x+\sqrt{t}y)\diff\gauss_{0,1}(y),
\end{equation}
and the analogue of the Doob $h$-transform in the continuous setting is 
\begin{equation}
\label{eq:doob_follmer}
\h(t,\x):=\heat_{\Time-t}\reltarget(\x).
\end{equation}
The F\"ollmer process is the solution to the stochastic differential equation over $t\in [0,\Time]$,
\begin{equation}
\label{eq:follmer}
\diff \fp_t=\nabla\log \heat_{\Time-t}\reltarget(\fp_t)\diff t+\diff\bm_t,
\end{equation}
where $(B_t)$ is a standard Brownian motion in $\R^{\Dim}$. Analogous to Corollary \ref{cor:time_marg}, the time marginals $(\flow_t)_{t\in [0,\Time]}$ of $(\fp_t)_{t\in [0,\Time]}$ satisfy 
\begin{equation}
\label{eq:time_marginal_follmer}
\flow_t(\x)=\h(t,\x)\gauss_{0,t}(\x),
\end{equation}
so, in particular, $\fp_{\Time}\sim\target$. The drift $\nabla\log \heat_{\Time-t}\reltarget$ is the continuous analogue of the intensity $\ints$,  and satisfies a Tweedie's formula
\begin{equation}
\label{eq:tweedie_follmer}
\nabla\log \heat_{\Time-t}\reltarget(\x)=\frac{\den(t,x)-\x}{\Time -t},
\end{equation}
where
\[
\den(t,x):=\E[\fp_{\Time}|\fp_t=\x].
\]
Thus, the drift $\nabla\log \heat_{\Time-t}\reltarget$ can be learned via objectives such as \eqref{eq:denoiser_obj_intro}. 
To verify \eqref{eq:tweedie_follmer} we use the fact (Lemma 4.1 in \cite{MR4797372}) that the distribution of $\fp_{\Time}|\fp_t $ is $y\mapsto \frac{\reltarget(y)\gauss_{\fp_t,\Time -t}(y)}{\heat_{\Time-t}\reltarget(\fp_t)}$, so
\begin{align*}
\E[\fp_{\Time}|\fp_t]=\int_{\R^{\Dim}}y \frac{\reltarget(y)\gauss_{\fp_t,\Time -t}(y)}{\heat_{\Time-t}\reltarget(\fp_t)}\diff y,
\end{align*}
while on the other hand, by integration by parts,
\begin{align*}
&\nabla\log \heat_{\Time-t}\reltarget(x)=\frac{1}{\heat_{\Time-t}\reltarget(x)}\int_{\R^{\Dim}}\frac{1}{\sqrt{\Time -t}}\nabla_y\reltarget(\x+\sqrt{\Time -t}y)\diff\gauss_{0,1}(y)\\
&=\frac{1}{\heat_{\Time-t}\reltarget(x)}\int_{\R^{\Dim}}y\frac{1}{\sqrt{\Time -t}}\reltarget(\x+\sqrt{\Time -t}y)\diff\gauss_{0,1}(y)=\frac{1}{\Time -t}\int_{\R^{\Dim}}(y-x)\frac{\reltarget(y)\diff\gauss_{x,\Time -t}(y)}{\heat_{\Time-t}\reltarget(x)}.
\end{align*}
Note that, by \eqref{eq:time_marginal_follmer} and \eqref{eq:doob_follmer},
\begin{equation}
\label{eq:tweedie_follmer_marg}
\nabla\log\flow_t(\x)=\nabla\log \heat_{\Time-t}\reltarget(\x)-\frac{\x}{t}=\frac{1}{\Time-t}\den(t,x)-\frac{\Time}{t(\Time-t)}x.
\end{equation}

Just like the Poisson-F\"ollmer process, the F\"ollmer process is a special case of a Schr\"odinger bridge, and the drift $\nabla\log \heat_{\Time-t}\reltarget$ is, up to a time change, the score one needs to learn in the time reversal of the Ornstein–Uhlenbeck process. 

\section{Experiments with synthetic data} \label{app:additional_numerics}

In this section we define the one-dimensional synthetic distributions considered in Section~\ref{sec:exp}, and provide additional details for training the denoiser with data from these problems. The synthetic experiments consist of the following distributions taken from~\cite{DBLP:journals/corr/abs-2505-05082}, which all have a truncated support of integers $\Space \subset \mathbb{N}_0^1$, where the support size $|\Space|$ is reported below for each problem.
\begin{enumerate} \itemsep0pt
    \item \textbf{Poisson}: $\mu(x) = \Poisson_{\lambda}(x)$ is a Poisson distribution with rate parameter $\lambda$. We take $\lambda = 5.0$ and $|\Space| = 40$. 
    \item \textbf{Poisson Mixture}: $\mu(x) = \sum_{i=1}^m w_i \Poisson_{\lambda_i}(x)$. We take $m = 2$ components with $w_1 = 0.1, w_2 = 0.9$, $\lambda_1 = 1, \lambda_2 = 100$ and $|\Space| = 140$.
    \item \textbf{Zero-Inflated-Poisson (ZIP)}: $\mu(x) = w_0 + (1 - w_0) e^{-\lambda}$ for $x = 0$ and $\mu(x) = (1 - w_0) \Poisson_{\lambda}(x)$ for $x > 0$. We take $w_0 = 0.7$, $\lambda = 5$, and $|\Space| = 50$.
    \item \textbf{Negative-Binomial-Mixture (NBM)}: $\mu(x) = \sum_{i=1}^m w_i \NB_{r_i,p_i}(x)$ where $r_i$ are the number of successes and $p_i$ are the success probabilities. We take $m = 2$ components with $r_1 = 1, r_2 = 10$, $p_1 = 0.9, p_2 = 0.1$, and $|\Space| = 150$.
    \item \textbf{Beta-Negative-Binomial (BNB)}: $\mu(x) = \int_0^1 \NB_{1,t}(x)\text{Beta}_{a,b}(t) \diff t$ where $\text{Beta}_{a,b}(t)$ is a Beta distribution with parameters $a,b$. We take $a = 1.5, b = 1.5$ and $r = 5$, and $|\Space| = 100$.
    \item \textbf{Zipf}: $\mu(x) = x^{-\alpha}/\zeta(\alpha)$ where $\zeta(\alpha)$ is Riemann zeta function. We take $\alpha = 1.7$ and $|\Space| = 50$.
    \item \textbf{Yule-Simon}: $\mu(x) = \rho \Gamma(x)\Gamma(\rho + 1)/\Gamma(x + \rho + 1),$ where $\Gamma$ is the Gamma function and $\rho \in \R_+$ is a positive constant. We take $\rho = 2.0$ and $|\Space| = 50$.
\end{enumerate}

For these experiments we use a residual MLP architecture for the denoiser with sinusoidal time-embeddings. The model has 3 hidden layers with 256 hidden dimensions, and 128 dimensional embeddings for time. We learn the network parameters using the Adam optimizer by training for $300$ epochs with a learning rate of $10^{-3}$, a batch size of $128$, weight decay of magnitude $10^{-5}$, gradient clipping, and taking an exponential moving average of the parameters for the reported results. 

\begin{table}[!ht]
\centering
\caption{Evaluation of the $W_1$ metric (lower is better) between the target distribution and generated samples from Binomial Flows in comparison to itDPDM~\cite{DBLP:journals/corr/abs-2505-05082}}
\label{tab:performance_comparison}
\begin{tabular}{lcc}
\midrule
\textbf{Problem} 
& \textbf{Binomial Flows} & \textbf{itDPDM} \\
\midrule
Poisson & $0.08 \pm 0.02$ & - \\
Poisson Mixture & $1.52 \pm 0.14$ & $\mathbf{0.99 \pm 0.15}$ \\
ZIP & $\mathbf{0.09 \pm 0.02}$ & $0.56 \pm 0.43$ \\
NBM & $1.74 \pm 0.30$ & $\mathbf{1.39 \pm 0.37}$ \\
BNB & $0.79 \pm 0.15$ & $\mathbf{0.67 \pm 0.23}$ \\
Zipf &  $\mathbf{0.18 \pm 0.02}$ & $0.48 \pm 0.13$ \\
Yule-Simon &  $\mathbf{0.11 \pm 0.01}$ & $0.14 \pm 0.03$ \\
\midrule
\end{tabular}
\end{table}

Figure~\ref{fig:likelihood} plots the log-likelihood of the predicted distribution over the support of the target distribution using identity~\eqref{eq:loglike}, in comparison to the true log of the probability mass function (PMF) $\log\mu(x)$. The predicted log-likelihood is computed using a Monte Carlo estimator with $10,000$ samples. The mean and standard error of the estimator are reported using the solid and dashed lines, respectively. We note that the estimated log-likelihood shows close agreement to the true PMF, especially in regions where the target distribution has higher probability mass, e.g., we note that growth in the standard error for the Yule-Simon distribution for larger $x$.

\begin{figure}[!ht]
\centering
\includegraphics[width=0.32\linewidth]{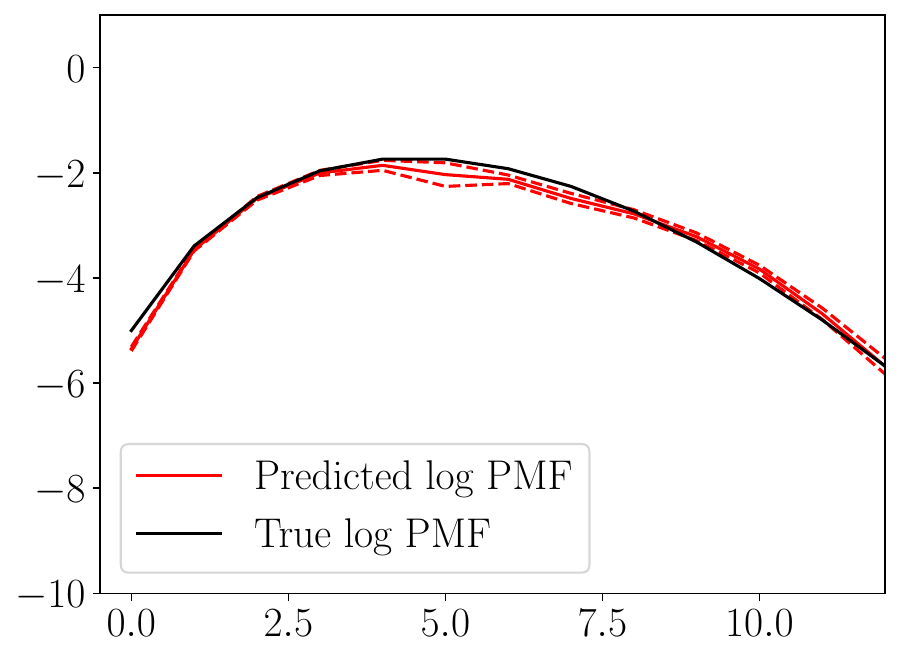}
\includegraphics[width=0.32\linewidth]{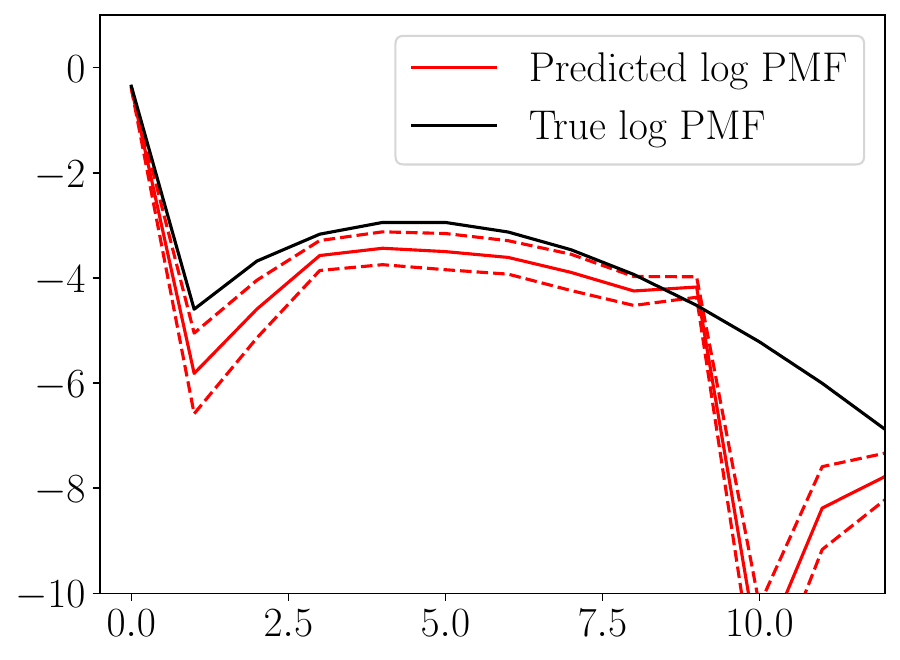}
\includegraphics[width=0.32\linewidth]{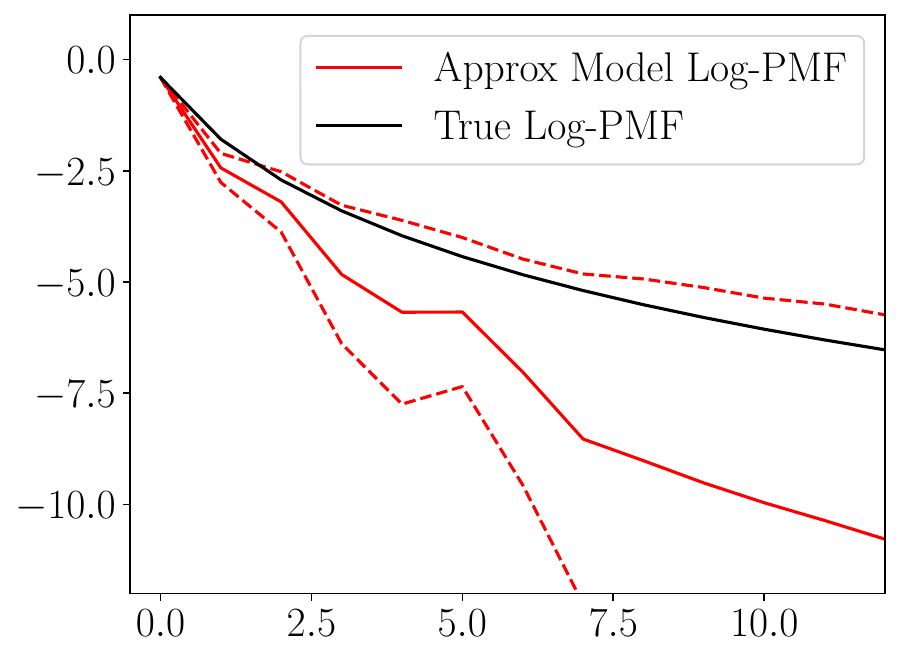}
\caption{True and estimated log-likelihood evaluations of the Poisson, Zero-Inflated Poisson, and Yule-Simon distributions \label{fig:likelihood}}
\end{figure}

\section{Experiments with imaging data}\label{sec:experiments}

\subsection{Preconditioning}
\label{subsec:preconditioning_appen}
In this section we provide further details on preconditioning. We follow the EDM framework~\citep{karras2022elucidating} where we write 
\begin{equation}
\label{eq:mf_appen}
\den_{\theta}(\x_t,t)=c_{\textnormal{skip}}(t)\x_t+c_{\textnormal{out}}(t)F_{\theta}(\widetilde{\x_t},t),\qquad \qquad\widetilde{x}_t=c_{\text{in}}(t) x_t + s_{\text{in}}.
\end{equation}
Following EDM's optimality principles, we derive scaling functions $c_{\text{in}}(t)$, $c_{\text{skip}}(t)$, $c_{\text{out}}(t)$ and loss weight $w^2(t)$ to satisfy three criteria: (i) $c_{\text{in}}(t)$ normalizes network inputs $\widetilde{x}_t$ to unit variance (and additional centering via $s_{\text{in}}$ in our case), (ii) $c_{\text{skip}}(t)$ minimizes $c_{\text{out}}(t)$ to prevent error amplification while $c_{\text{out}}(t)$ ensures unit variance training targets, and (iii) $w^2(t)$ maintains constant expected loss across time levels. Just for the purpose of this derivation we assume that the coordinates of $\cp_1\sim\target$ are i.i.d. with mean $\mu_{\textnormal{data}}$ and variance $\sigma^2_{\textnormal{data}}$.

\subsubsection{Binomial noise}
In this section we derive the EDM scaling for the Binomial noise used in our Binomial flow algorithm. 
\begin{claim}
\label{cl:mean_var_bino}
Fix $\idx\in [\Dim]$ and $t\in [0,1]$. Then,
\begin{equation}
\label{eq:cond_min_x_xt}
\E[\cp_t^{\idx}|\cp_1]=t\cp_1^{\idx},
\end{equation}
\begin{equation}
\label{eq:var_x_xt}
\var[\cp_t^{\idx}|\cp_1]= t(1-t)\cp_1^{\idx},
\end{equation}
\begin{equation}
\label{eq:x_xt_square}
\E[(\cp_t^{\idx})^2|\cp_1]=t(1-t)\cp_1^{\idx}+t^2(\cp_1^{\idx})^2,
\end{equation}
\begin{equation}
\label{eq:xt}
\E[\cp_t^{\idx}]=t\mu_{\textnormal{data}},
\end{equation}
\begin{equation}
\label{eq:var_xt}
\var[\cp_t^{\idx}]=\mu_{\textnormal{data}}t(1-t)+\sigma^2_{\textnormal{data}}t^2,
\end{equation}
\begin{equation}
\label{eq:xt_sq}
\E[(\cp_t^{\idx})^2]=\mu_{\textnormal{data}}t(1-t)+\sigma^2_{\textnormal{data}}t^2+t^2\mu_{\textnormal{data}}^2,
\end{equation}
\begin{equation}
 \label{eq:e_x1_xt} 
\E[\cp_t^{\idx}\cp_1^{\idx}]=t(\sigma^2_{\textnormal{data}}+\mu_{\textnormal{data}}^2),
\end{equation}
\begin{equation}
\label{eq:cov_x1_xt}
\cov\left(\cp_1^{\idx},\cp_t^{\idx}\right)=t\sigma^2_{\textnormal{data}}.
\end{equation}
\end{claim}
\begin{proof}
We recall that, by Lemma \ref{lem:bridge},$\cp_t^{\idx}|\cp_1\sim \B_{\x_1^{\idx},1}$, and that the mean and variance of $\B_{n,\alpha}$ are $n\alpha$ and $n\alpha(1-\alpha)$, respectively. Equation \eqref{eq:cond_min_x_xt} and Equation \eqref{eq:var_x_xt} then follow immediately. Equation \eqref{eq:x_xt_square} follows from the combination of \eqref{eq:cond_min_x_xt} and \eqref{eq:var_x_xt}. Equation \eqref{eq:xt} follows by taking expectation in \eqref{eq:cond_min_x_xt}. Equation \eqref{eq:var_xt} follows from the law of total variance,
\[
\var[\cp_t^{\idx}]=\E[\var[\cp_t^{\idx}|\cp_1]]+\var[\E[\cp_t^{\idx}|\cp_1]]=\mu_{\text{data}}t(1-t)+\sigma^2_{\text{data}}t^2.
\]
Equation \eqref{eq:xt_sq} follows from the combination of \eqref{eq:xt} and \eqref{eq:var_xt}. Equation \eqref{eq:e_x1_xt} follows from \eqref{eq:cond_min_x_xt} and \eqref{eq:xt_sq} as
\begin{align*}
\E[\cp_1^{\idx}\cp_t^{\idx}]=\E[\cp_1^{\idx}\E[\cp_t^{\idx}|\cp_1^{\idx}]]=t\E[(\cp_1^{\idx})^2]=t(\sigma^2_{\textnormal{data}}+\mu_{\textnormal{data}}^2).
\end{align*}
Equation \eqref{eq:cov_x1_xt} follows from \eqref{eq:e_x1_xt}  and \eqref{eq:xt} as 
\begin{align*}
&\cov\left(\cp_1^{\idx},\cp_t^{\idx}\right)=\E[\cp_1^{\idx}\cp_t^{\idx}]-\E[\cp_1^{\idx}]\E[\cp_t^{\idx}]=t(\sigma^2_{\textnormal{data}}+\mu_{\textnormal{data}}^2)-t\mu_{\textnormal{data}}^2=t\sigma^2_{\textnormal{data}}.
\end{align*}
\end{proof}
We now derive the expression for $c_{\text{in}}(t)$ which is chosen to keep the variance of the input $\widetilde{\x}_t$ constant.
\begin{claim}[Derivation of $c_{\text{in}}$]
\label{cl:cin}
If
\begin{equation}
\label{eq:c_in}
c_{\textnormal{in}}(t)=\dfrac{1}{\sqrt{\mu_{\textnormal{data}} t(1-t) + \sigma^2_{\textnormal{data}} t^2 }}
\end{equation}
then $\var[c_{\textnormal{in}}(t)\cp_t^{\idx}]=1$ for all $t\in [0,1]$ and $\idx\in [\Dim]$.
\end{claim}
\begin{proof}
Follows immediately from \eqref{eq:var_xt}.
\end{proof}
Next we turn to the derivation of the relation between $c_{\textnormal{out}}$ and $c_{\textnormal{skip}}$. With the parametrization \eqref{eq:mf_appen} the objective \eqref{eq:denoiser_obj_quad_exp} reads
\begin{equation}
\label{eq:train_loss_appen}  
\begin{split}
&\int_0^1\int_{\Space}\w(t)\left|\x_1-\den_{\theta}(\x_t,t)\right|^2 \flow_{t|1}(\x_t|\x_{1})\target(\x_1)\diff t\\
&=\int_0^1\int_{\Space}\w(t)c_{\textnormal{out}}^2(t)\left|F_{\theta}(\widetilde{\x_t},t)-F_{\textnormal{target}}(\x_1,\x_t,t)\right|^2\flow_{t|1}(\x_t|\x_1)\target(\x_1)\diff t,
\end{split}
\end{equation}
where 
\begin{equation}
\label{eq:F_target}
F_{\textnormal{target}}(\x_1,\x_t,t):=\frac{1}{c_{\textnormal{out}}(t)}\left(\x_1-c_{\textnormal{skip}}(t)\x_t\right).
\end{equation}

We will choose $c_{\textnormal{out}}$ and $c_{\textnormal{skip}}$ so that the variance of $F_{\textnormal{target}}$ is constant in time.
\begin{claim}[Relation between $c_{\textnormal{out}}$ and $c_{\textnormal{skip}}$]
\label{cl:cout_sckip}
Let 
\begin{equation}
\label{eq:cout_sckip}
c_{\textnormal{out}}^2(t)=\sigma^2_{\textnormal{data}}+c_{\textnormal{skip}}^2(t)\left(\mu_{\textnormal{data}}t(1-t)+\sigma^2_{\textnormal{data}} t^2\right)-2c_{\textnormal{skip}}(t)\sigma^2_{\textnormal{data}} t.
\end{equation}
Then, for all $t\in [0,1]$ and $\idx\in [\Dim]$,
\begin{equation}
\label{eq:var_F}    
\var[F_{\textnormal{target}}^{\idx}(\cp_1,\cp_t,t)]=1,
\end{equation}
where $F_{\textnormal{target}}^{\idx}$ is the $\idx$th coordinate of $F_{\textnormal{target}}$ for $\idx\in [\Dim]$. 
\end{claim}
\begin{proof}
The result follows by combining  \eqref{eq:F_target}, \eqref{eq:var_xt}, and \eqref{eq:cov_x1_xt}.
\end{proof}
Claim \ref{cl:cout_sckip} provides the relation between $c_{\textnormal{out}}$ and $c_{\textnormal{skip}}$ so it suffices to determine the value of one of them. We do so by choosing $c_{\textnormal{skip}}$ to minimize $c_{\textnormal{out}}$. 
\begin{claim}[Best $c_{\textnormal{skip}}$]
\label{cl:c_skip}
For fixed $c_{\textnormal{out}}$ the minimizer of \eqref{eq:cout_sckip} is
\begin{equation}
    \label{eq:c_skip}
c_{\textnormal{skip}}(t):=\frac{\sigma^2_{\textnormal{data}}}{\mu_{\textnormal{data}}(1-t)+\sigma^2_{\textnormal{data}}t}.
\end{equation}
\end{claim}
\begin{proof}
As a function of $c_{\textnormal{skip}}(t)$ the expression in \eqref{eq:cout_sckip} is convex so the result follows by equating the its derivative with respect to $c_{\textnormal{skip}}(t)$ to zero.
\end{proof}
Once $c_{\textnormal{skip}}$ is determined, the expression for $c_{\textnormal{out}}$ follows by substituting \eqref{eq:c_skip} into \eqref{eq:cout_sckip}. 

Next we determine the value of $\w$. We do so by requiring the loss to stay constant in time at initialization $F_{\theta}=0$. Note that when $F_{\theta}=0$, Equation \eqref{eq:train_loss_appen} reads
\begin{equation}
\label{eq:train_loss_appen_int}  
\int_0^1\int_{\Space}\w(t)\left|\x_1-c_{\textnormal{skip}}\x_t\right|^2\flow_{t|1}(\x_t|\x_{1})\target(\x_1)\diff t.
\end{equation}

\begin{claim}
\label{eq:loss_constant}
Let 
\begin{equation}
\label{eq:w}   
\w(t)=\frac{1}{c_{\textnormal{out}}^2(t)+\left(\mu_{\textnormal{data}}-c_{\textnormal{skip}}(t)t\mu_{\textnormal{data}}\right)^2}.
\end{equation}
Then, for all $t\in [0,1]$,
\begin{equation}
\label{eq:train_loss_appen_int_const}  
\int_{\Space}\w(t)\left|\x_1-c_{\textnormal{skip}}\x_t\right|^2\flow_{t|1}(\x_t|\x_{1})\target(\x_1)=\Dim.
\end{equation}
\end{claim}
\begin{proof}
Equation \eqref{eq:train_loss_appen_int_const} reads 
\begin{align*}
&\int_{\Space}\w(t)\left|\x_1-c_{\textnormal{skip}}\x_t\right|^2\flow_{t|1}(\x_t|\x_{1})\target(\x_1)=\w(t)\sum_{\idx=1}^{\Dim}\E[(\cp_1^{\idx}-c_{\textnormal{skip}}(t)\cp_t^{\idx})^2].
\end{align*}
On the other hand,
\begin{align*}
&\E[(\cp_1^{\idx}-c_{\textnormal{skip}}(t)\cp_t^{\idx})^2]=\var[\cp_1^{\idx}-c_{\textnormal{skip}}(t)\cp_t^{\idx}]+\E[\cp_1^{\idx}-c_{\textnormal{skip}}(t)\cp_t^{\idx}]^2. 
\end{align*}
By Claim \ref{cl:cout_sckip}, the first term reads
\begin{align*}
\var[\cp_1^{\idx}-c_{\textnormal{skip}}(t)\cp_t^{\idx}]=c_{\textnormal{out}}^2(t)\var[F_{\textnormal{target}}(\cp_1,\cp_t,t)]=c_{\textnormal{out}}^2(t),
\end{align*}
while, by \eqref{eq:xt}, the second term reads
\begin{align*}
\E[\cp_1^{\idx}-c_{\textnormal{skip}}(t)\cp_t^{\idx}]^2=\left(\mu_{\textnormal{data}}-c_{\textnormal{skip}}(t)t\mu_{\textnormal{data}}\right)^2.
\end{align*}
We conclude that 
\[
\int_{\Space}\w(t)\left|\x_1-c_{\textnormal{skip}}\x_t\right|^2\flow_{t|1}(\x_t|\x_{1})\target(\x_1)=\Dim\,\w(t)\left[c_{\textnormal{out}}^2(t)+\left(\mu_{\textnormal{data}}-c_{\textnormal{skip}}(t)t\mu_{\textnormal{data}}\right)^2\right],
\]
which proves the result.
\end{proof}
Next we show how to choose $b_{\text{skip}}$  and $b_{\text{out}}$ so that
$\mathrm{b}(t) := b_{\text{skip}}(t) \cp_t + b_{\text{out}}(t) \mu_{\text{data}}$ minimizes the expected squared error $\loss_t(\mathrm{b}) = \E[|\cp_1 - \mathrm{b}(t)|^2]$ at each time $t$. 
\begin{claim}
\label{cl:opt_b_bin}
The minimizer $\mathrm{b}(t) := b_{\textnormal{skip}}(t) \cp_t + b_{\textnormal{out}}(t) \mu_{\textnormal{data}}$ of  $\loss_t(\mathrm{b}) = \E[|\cp_1 - \mathrm{b}(t)|^2]$ is
\begin{equation}
\label{eq::opt_b_bin}
b_{\textnormal{skip}}(t)=c_{\textnormal{skip}}(t)=\frac{\sigma^2_{\textnormal{data}}}{\mu_{\textnormal{data}}(1-t)+\sigma^2_{\textnormal{data}}t},\qquad b_{\textnormal{out}}(t) = 1- t b_{\textnormal{skip}}(t).
\end{equation}
\end{claim}
\begin{proof}
Fix $\idx\in [\Dim]$. Equation \eqref{eq:xt_sq} gives
\begin{align*}
\E[(\cp_1^{\idx})^2]]=\sigma^2_{\textnormal{data}}+\mu_{\textnormal{data}}^2,
\end{align*}
while, by \eqref{eq:e_x1_xt}, 
\begin{align*}
\E[\cp_1^{\idx}\mathrm{b}^{\idx}(t)]= b_{\text{skip}}(t) \E[\cp_1^{\idx}  \cp_t^{\idx}] + b_{\text{out}}(t) \mu_{\text{data}}\E[\cp_1^{\idx}]=b_{\text{skip}}(t)t(\sigma^2_{\textnormal{data}}+\mu_{\textnormal{data}}^2)+ b_{\text{out}}(t) \mu_{\text{data}}^2.
\end{align*}
By  \eqref{eq:xt} and \eqref{eq:xt_sq},
\begin{align*}
\E[(\mathrm{b}^{\idx}(t))^2]&=b_{\text{skip}}^2(t)\E[(\cp_t^{\idx})^2]+2b_{\text{skip}}(t)b_{\text{out}}(t)\mu_{\text{data}}\E[\cp_t^{\idx}]+b_{\text{out}}^2(t)\mu_{\text{data}}^2\\
&=b_{\text{skip}}^2(t)\left(\mu_{\textnormal{data}}t(1-t)+\sigma^2_{\textnormal{data}}t^2+t^2\mu_{\textnormal{data}}^2\right)+2b_{\text{skip}}(t)b_{\text{out}}(t)t\mu_{\text{data}}^2+b_{\text{out}}^2(t)\mu_{\text{data}}^2,
\end{align*}
so putting everything together,
\begin{align*}
&\E[|\cp_1^{\idx} - \mathrm{b}^{\idx}(t)|^2]=\left\{\sigma^2_{\textnormal{data}}+\mu_{\textnormal{data}}^2\right\}-2\left\{b_{\text{skip}}(t)t(\sigma^2_{\textnormal{data}}+\mu_{\textnormal{data}}^2)+ b_{\text{out}}(t) \mu_{\text{data}}^2\right\}\\
&+\left\{b_{\text{skip}}^2(t)\left(\mu_{\textnormal{data}}t(1-t)+\sigma^2_{\textnormal{data}}t^2+t^2\mu_{\textnormal{data}}^2\right)+2b_{\text{skip}}(t)b_{\text{out}}(t)t\mu_{\text{data}}^2+b_{\text{out}}^2(t)\mu_{\text{data}}^2\right\}.
\end{align*}
Differentiating the error with respect to $b_{\text{skip}}(t)$ we obtain
\begin{align*}
&-2t(\sigma^2_{\textnormal{data}}+\mu_{\textnormal{data}}^2)+2b_{\text{skip}}(t)\left(\mu_{\textnormal{data}}t(1-t)+\sigma^2_{\textnormal{data}}t^2+t^2\mu_{\textnormal{data}}^2\right)+2b_{\text{out}}(t)t\mu_{\text{data}}^2 = 0,
\end{align*}
and hence  
\begin{align*}
&b_{\text{skip}}(t)=\frac{t(\sigma^2_{\textnormal{data}}+\mu_{\textnormal{data}}^2)-b_{\text{out}}(t)t\mu_{\text{data}}^2}{\mu_{\textnormal{data}}t(1-t)+\sigma^2_{\textnormal{data}}t^2+t^2\mu_{\textnormal{data}}^2}.
\end{align*}
Differentiating the error with respect to $b_{\text{out}}(t)$ we obtain
\begin{align*}
&-2\mu_{\text{data}}^2 + 2b_{\text{skip}}(t)t\mu_{\text{data}}^2+2b_{\text{out}}(t)\mu_{\text{data}}^2 = 0,
\end{align*}
and hence 
\begin{align*}
b_{\text{out}}(t) = 1- t b_{\text{skip}}(t).
\end{align*}
Plugging this in the expression for $b_{\text{skip}}$ we obtain
\begin{align*}
b_{\text{skip}}(t)&=\frac{t(\sigma^2_{\textnormal{data}}+\mu_{\textnormal{data}}^2)-(1- t b_{\text{skip}}(t))t\mu_{\text{data}}^2}{\mu_{\textnormal{data}}t(1-t)+\sigma^2_{\textnormal{data}}t^2+t^2\mu_{\textnormal{data}}^2}\\
&=\frac{\sigma^2_{\textnormal{data}}+  tb_{\text{skip}}(t)\mu_{\text{data}}^2}{\mu_{\textnormal{data}}(1-t)+\sigma^2_{\textnormal{data}}t+t\mu_{\textnormal{data}}^2},
\end{align*}
and rearranging the terms and solving for $b_{\text{skip}}$ gives us 
\begin{align*}
&b_{\text{skip}}(t)=\frac{\sigma^2_{\textnormal{data}} }{\mu_{\textnormal{data}}(1-t)+\sigma^2_{\textnormal{data}}t}.
\end{align*}
\end{proof}

\begin{table}[H]%
\centering
\caption{Binomial flow EDM ($T = 1$). Here $\mu_{\text{data}}$ and $\sigma^2_{\text{data}}$ are the mean and variance of the unscaled images $[0, 255]$, and $\varepsilon_{\text{cin}} = 0.01$ in our experiments.}
\label{tab:pf-edm}
\begin{tabular}{@{}ll@{}}
\toprule
\textbf{Symbol} & \textbf{Formula} \\
\midrule
$\sigma$ & $-\log(t + \varepsilon_\text{noise})$ \\
$t$ & $\exp(-\sigma) - \varepsilon_\text{noise}$ \\
\midrule
$p(\sigma)$ & $\propto \mathcal{N}(\sigma;\mu_\sigma, \gamma^2_\sigma)\mathbf{1}_{[0,-\log(\varepsilon_\text{noise})]}(\sigma)$ \\
$\trans_{t|1}(\x_t \mid x_1)$ & $\B_{\x_1, t}$ \\
\midrule
$c_{\text{in}}(t)$ & $\dfrac{1}{\sqrt{\mu_{\text{data}} t(1-t) + \sigma^2_{\text{data}} t^2 + \varepsilon_{\text{cin}}}}$ \\[2ex]
$s_{\text{in}}$ & $-\dfrac{\mu_{\text{data}}}{\sqrt{\sigma^2_{\text{data}}}}$ \\[2ex]
$c_{\text{skip}}(t)$ & $\dfrac{\sigma^2_{\text{data}}}{\mu_{\text{data}}(1-t) + \sigma^2_{\text{data}} t}$ \\[2ex]
$c_{\text{out}}(t)$ & $\sqrt{\dfrac{\sigma^2_{\text{data}} \mu_{\text{data}} (1-t)}{\mu_{\text{data}}(1-t) + \sigma^2_{\text{data}} t}}$ \\[2ex]
$w^2(t)$ & $\dfrac{1}{c_{\text{out}}(t)^2 + (\mu_{\text{data}} - c_{\text{skip}}(t) \mu_{\text{data}} t)^2 + \varepsilon_{\text{cin}}}$ \\[2ex]
\midrule
$\widetilde{x}_t$ & $c_{\text{in}}(t) x_t + s_{\text{in}}$ \\[1ex]
$\den(x_t, t)$ & $c_{\text{skip}}(t) x_t + c_{\text{out}}(t) F_\theta(\widetilde{x}_t, \sigma)$ \\[1ex]
$\loss_t(\den)$ & $\w^2(t) \|\den(\x_t,t) - x_1\|^2$ \\
\bottomrule
\end{tabular}
\end{table}

\subsection{Additional figures}
\label{subsec:add_fig}

\ \ 

\begin{figure}[H]
\centering
\includegraphics[width=0.45\linewidth]{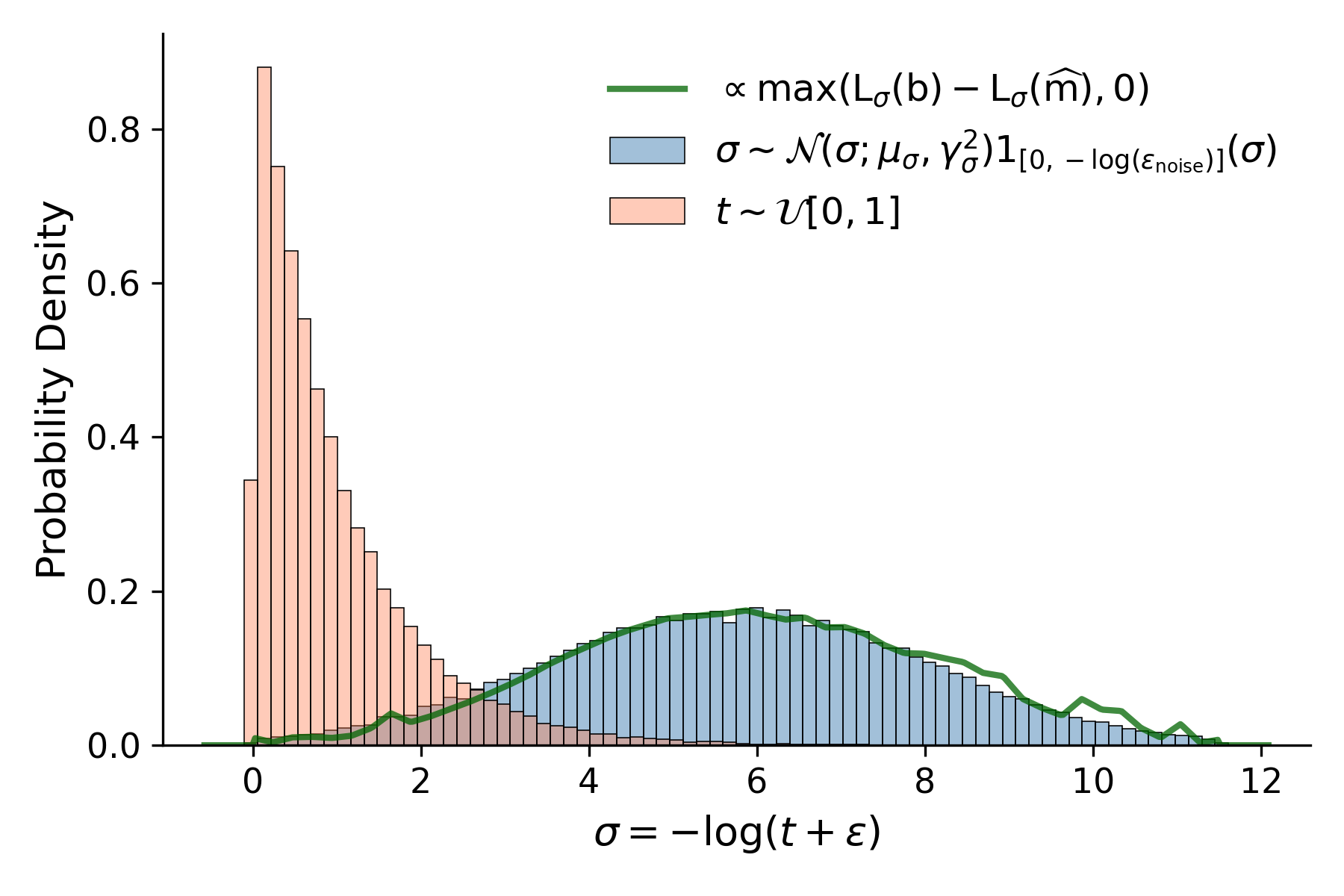}
\caption{Noise sampling distributions. The normalized improvement curve $\loss_\sigma(\mathrm{b}) - \loss_\sigma(\widehat{\den})$ (solid) compared to the induced distributions from uniform time sampling $t \sim \mathcal{U}[0,1]$ and matched Gaussian noise sampling.}\label{fig:time}
\end{figure}

\begin{figure}[H]
\centering
\includegraphics[width=0.5\linewidth]{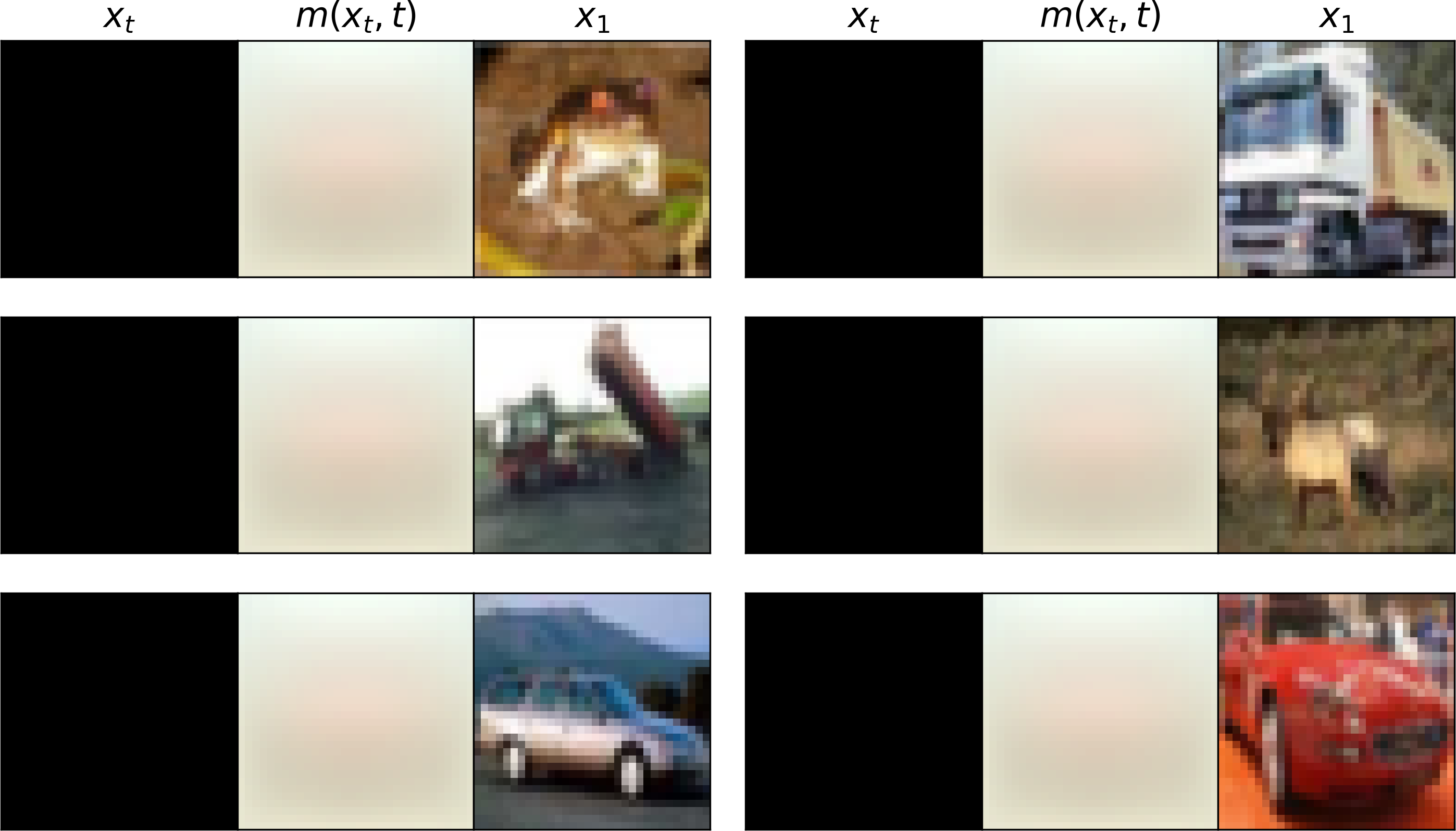}
\includegraphics[width=0.5\linewidth]{denoiser_t0.001.png}
\includegraphics[width=0.5\linewidth]{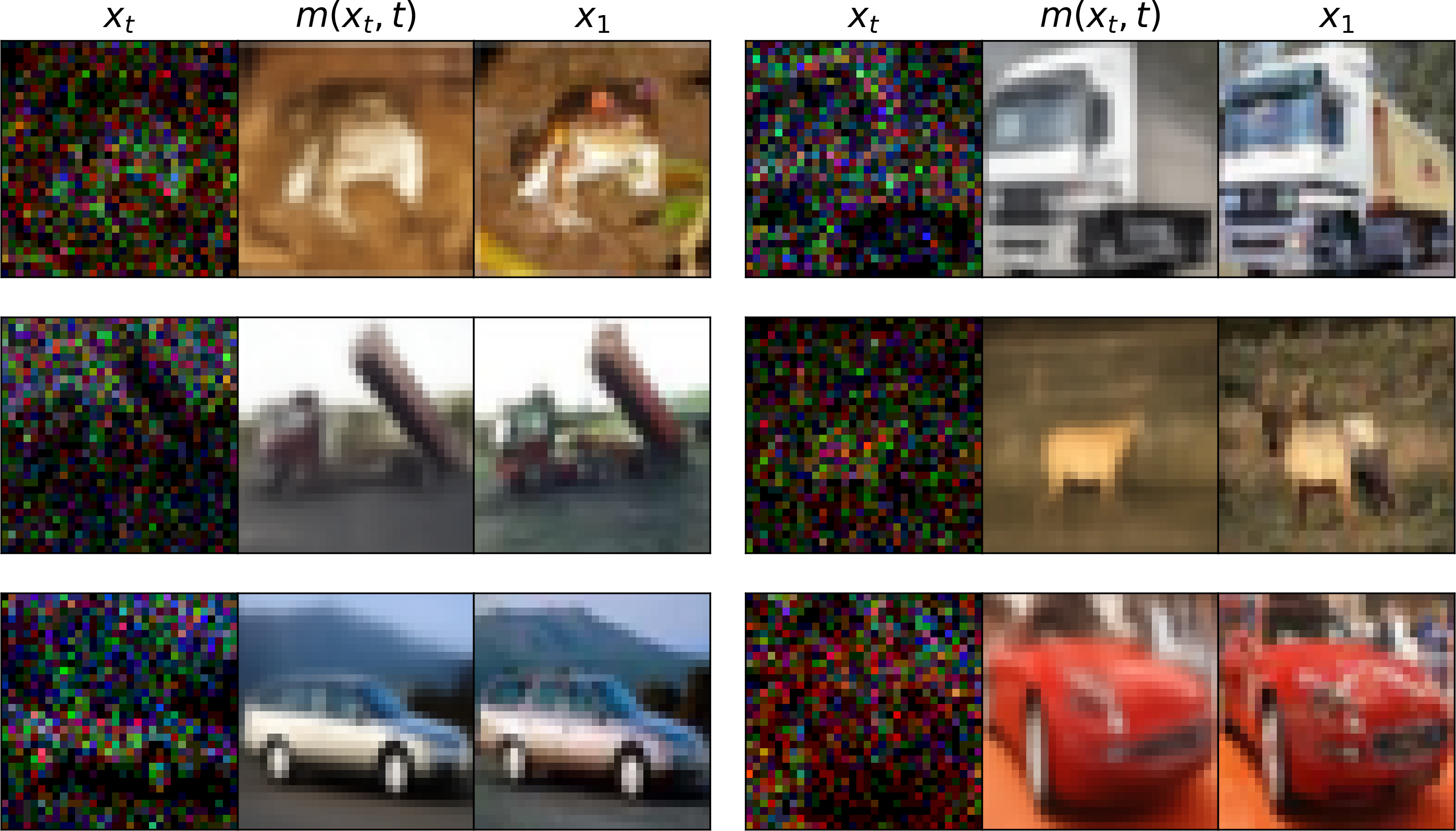}
\includegraphics[width=0.5\linewidth]{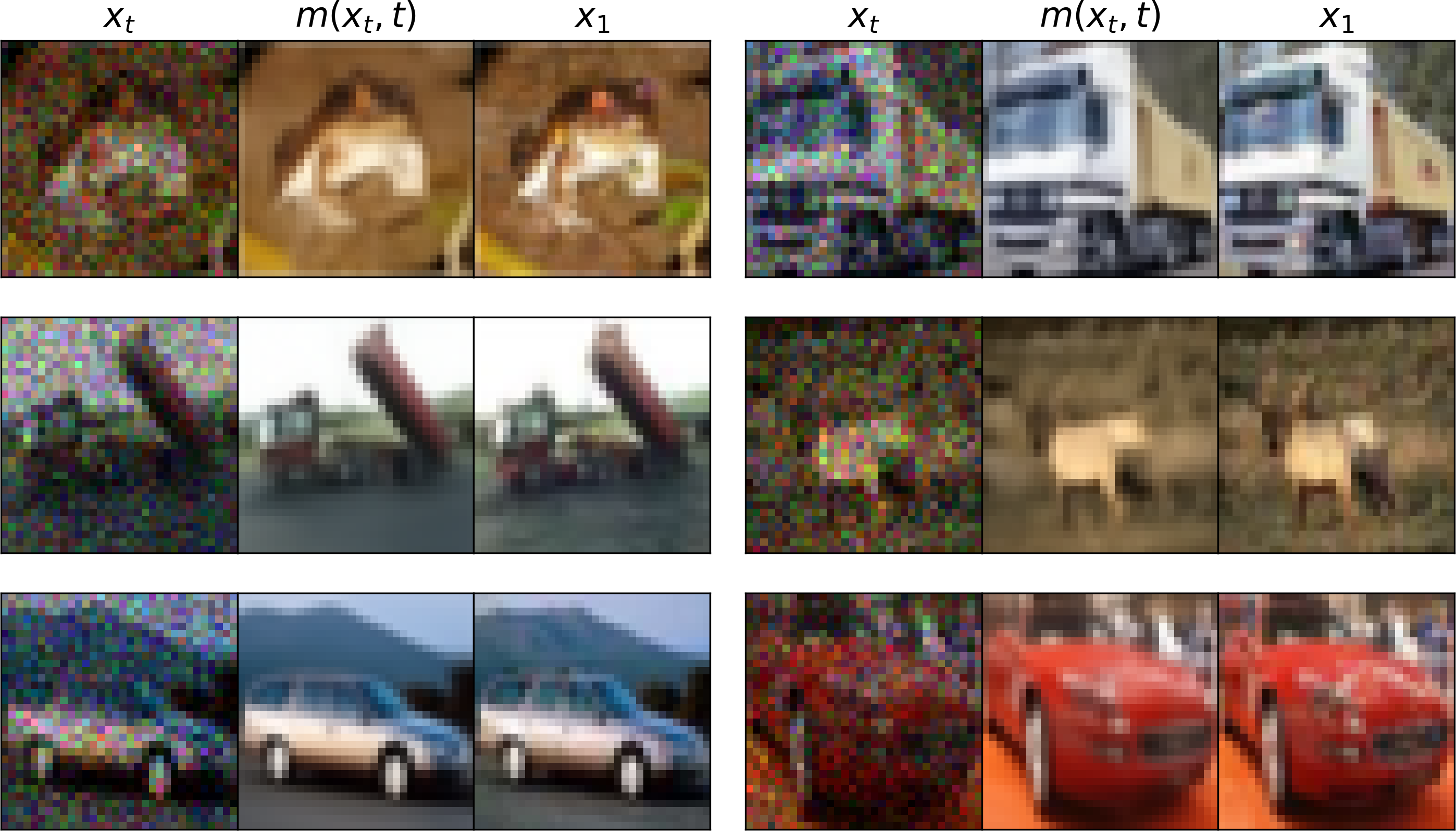}
\caption{Denoiser Poisson-Föllmer EDM samples for $t=0,0.001,0.01,0.1$.}\label{fig:denoiser_all}
\end{figure}

\begin{figure}[H]
\centering
\includegraphics[width=0.7\linewidth]{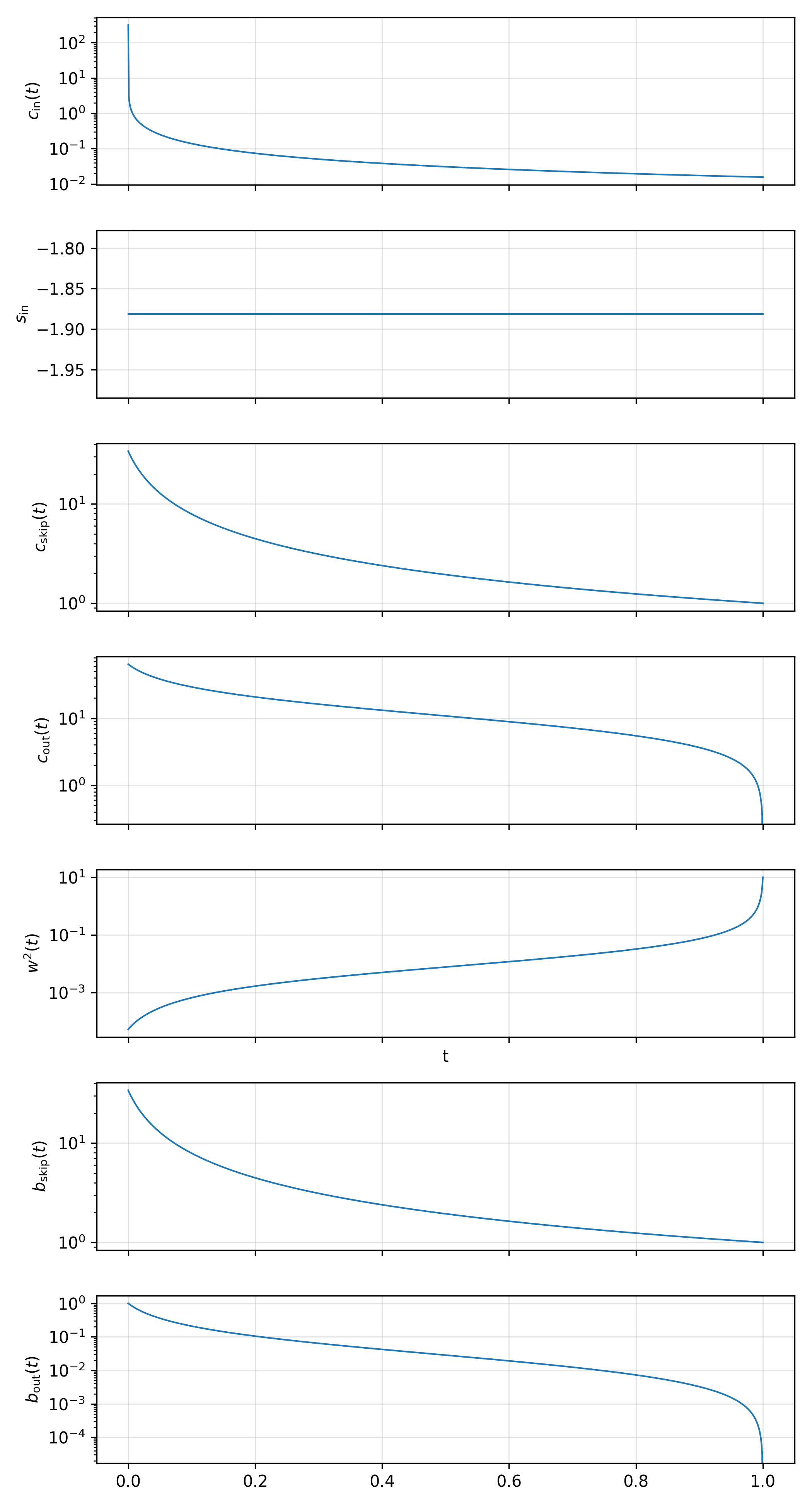}
\caption{Scaling functions for CIFAR10.}\label{fig:scalings}
\end{figure}

\ \

\phantom{.}
\end{document}